\documentclass{article} % For LaTeX2e
\usepackage{iclr2027_conference,times}

\usepackage[utf8]{inputenc} % allow utf-8 input
\usepackage[T1]{fontenc}    % use 8-bit T1 fonts
\usepackage{url}            % simple URL typesetting
\usepackage{booktabs}       % professional-quality tables
\usepackage{amsfonts}       % blackboard math symbols
\usepackage{nicefrac}       % compact symbols for 1/2, etc.
\usepackage{microtype}      % microtypography
\usepackage[dvipsnames]{xcolor}
\usepackage{xfrac}          % fraction format
\usepackage{amsmath}\allowdisplaybreaks
\usepackage{bm}
\usepackage{amssymb}
\numberwithin{equation}{section}
\usepackage[colorinlistoftodos]{todonotes}
\usepackage{comment}
\usepackage{wrapfig}
\usepackage{tcolorbox}
\usepackage{enumitem}
\usepackage{graphicx}
\usepackage{subcaption}

\usepackage{amsthm}
\usepackage{float}
\usepackage{thmtools, thm-restate}
\usepackage{hyperref}
\usepackage[nameinlink,noabbrev]{cleveref}

\newtheorem{proposition}{Proposition}[section]
\newtheorem{definition}{Definition}[section]

\crefname{theorem}{theorem}{theorems}
\Crefname{theorem}{Theorem}{Theorems}
\crefname{proposition}{proposition}{propositions}
\Crefname{proposition}{Proposition}{Propositions}
\crefname{definition}{Definition}{Definitions}
\Crefname{definition}{Definition}{Definitions}
\crefname{assumption}{Assumption}{Assumptions}
\Crefname{assumption}{Assumption}{Assumptions}

\usepackage{pifont}% http://ctan.org/pkg/pifont
\definecolor{myblue}{RGB}{0,0,150}
\definecolor{myred}{RGB}{120,50,50}
\definecolor{mygray}{RGB}{90,90,110}

\hypersetup{
	colorlinks=true,
	linkcolor=myred,
	citecolor=myred,
	urlcolor=mygray,
}

\title{Interference Beyond Geometry\\ in Concept Extraction}

\author{%
  Valérie Costa\\
  University of Alberta, Amii\\
  Harvard University\\
  \And
  Bahareh Tolooshams\\
  University of Alberta, Amii\\
  Canada CIFAR AI Chair\\
}

\newcommand{\R}{\mathbb{R}} %real numbers
\newcommand{\E}{\mathbb{E}} %expectation

\newcommand{\W}{{\bm W}}

\newcommand{\A}{{\bm A}}

\newcommand{\D}{{\bm D}}

\newcommand{\dvec}{{\bm d}}
\newcommand{\gvec}{{\bm g}}

\newcommand{\eye}{{\bm I}}

\newcommand{\bvec}{{\bm b}}

\newcommand{\w}{{\bm w}}
\newcommand{\z}{{\bm z}}
\newcommand{\x}{{\bm x}}

\newcommand{\res}{{\bm r}}

\iclrfinalcopy % Uncomment for camera-ready version, but NOT for submission.
\begin{document}
\maketitle

%%%%%%%%%%%%%%%%%%%%%%%%%%%%%%%%%%
%%%%%%%%%%%%%%%%%%%%%%%%%%%%%%%%%%
%
\begin{abstract}
Interference is commonly treated as geometric overlap between learned features. We introduce \emph{effective interference}, which combines feature geometry and code statistics to capture realized interactions, distinguishing constructive from destructive interference and frequent weak interactions from rare strong ones. Under local fixed-support assumptions, we characterize how architectural constraints shape interference through four mechanisms: feature orthogonalization, bias compensation, gain adaptation, and encoder-decoder separation. Experiments with sparse autoencoders show that constrained architectures selectively reduce overlap among co-active features, while bias, gain, and encoder freedom allow constructive cross-contributions to remain. Together, these results show that interference in learned representations depends not only on feature geometry, but also on how features are used and on the architecture
that produces their codes.
\end{abstract}

%%%%%%%%%%%%%%%%%%%%%%%%%%%%%%%%%%
%%%%%%%%%%%%%%%%%%%%%%%%%%%%%%%%%%
%
\section{Introduction}\label{sec:intro}
Mechanistic interpretability~\citep{sharkey2025open,bereska2024mechanistic,conmy2023towards} seeks to understand the internal mechanisms of large neural networks. Central to this goal is understanding how structure in data is encoded within these models. Two complementary hypotheses have emerged as a foundation for this effort: the \emph{linear representation hypothesis}, in which concepts are represented as directions in activation space~\citep{mikolov-etal-2013-linguistic,park2024the,pmlr-v235-jiang24d,costa2025from}, and the \emph{superposition hypothesis}, which proposes that sparse representations allow neural networks to encode more concepts than the dimensionality of their hidden representations~\citep{elhage2022toymodelssuperposition,arora2018linear,klindt2025superposition}.

Under superposition, an overcomplete collection of concept directions cannot be mutually orthogonal and must overlap. This overlap is commonly called \emph{interference}: an inevitable cost of gaining representational capacity at the expense of feature separation~\citep{elhage2022toymodelssuperposition,bereska2024mechanistic,stevinson2026adversarial,adampoly}. Existing measures treat interference geometrically through pairwise cosine similarity~\citep{kong2026expand,gong2026signal} neuron capacity~\citep{adampoly}, or mutual coherence~\citep{tropp2004greed}, motivating methods that promote orthogonality among extracted concepts~\citep{korznikov2025ortsae,miller2026superpositioninterferenceisolatedinterventions,till_2024_saes_true_features, costa2025from}.

However, recent observations challenge this view: features with sparse or weakly correlated codes may overlap with little interference~\citep{adampoly,gurnee2023finding}, whereas correlated concepts may exploit overlap constructively~\citep{prieto2026from}. Overlap may also encode meaningful rather than incidental relationships~\citep{ivanov2026spectral,kantamneni2025language,gurnee2024language}, including cyclic structures such as days of the week and months of the year~\citep{engels2025not,stevinson2026adversarial}. Together, these observations suggest that a geometric measure of interference is incomplete without accounting for how concepts are used across the data.

Codes provide precisely this missing data-dependent information: unlike the concept dictionary, which is shared across inputs, a code specifies for each input which concepts are active and at what magnitude. Their \emph{support} captures when concepts co-activate and therefore which geometric overlaps are realized. In contrast, a pairwise geometric measure assigns the same weight to an overlap regardless of whether the corresponding concepts frequently co-activate or never do. Their \emph{magnitude}, in turn, captures how much a realized overlap contributes. Even when two concepts co-activate, the strength of their interaction depends on their activation magnitudes.

Interference is therefore fundamentally a joint property of concept geometry and code statistics (\Cref{fig:figinterference}). This distinction is particularly important for sparse autoencoders (SAEs) used in mechanistic interpretability~\citep{cunningham2023sparse,gao2025scaling,bussmann2024batchtopk,rajamanoharan2024jumping,zhu2026abstopk}: codes are not observed directly, but inferred from model activations by an encoder. How the encoder is parameterized, through weight tying, bias, or relaxed norm constraints, therefore affects the interference it realizes. We summarize our main contributions below:

\begin{itemize}[leftmargin=4mm]
\setlength\itemsep{0.1em}
\item \textbf{Effective interference.} We define \emph{effective interference} from pairwise cross-contributions in the reconstruction, combining feature geometry with how the corresponding codes are used across the data. Its factorization into orientation, co-activation frequency, and conditional magnitude separates the geometric and statistical factors underlying realized interactions.
\item \textbf{Encoder-dependent interference regimes.} Under local fixed-support conditions, we derive an encoder-consistency balance among self-response mismatch, directional incoming cross-talk, bias compensation, and residual coupling, and show how architectural degrees of freedom accommodate this balance through four non-exclusive routes: local orthogonalization, conditional bias compensation, encoder-gain adaptation, and encoder-decoder duality.
\item \textbf{Controlled validation in SAEs.}  We validate our theoretical predictions across SAE families trained on Pythia-160M activations \citep{biderman2023pythiasuiteanalyzinglarge}, showing that architectural constraints give rise to distinct and consistent interference regimes through orthogonalization, bias, gain, and encoder-decoder freedom. Similar patterns appear in SAEs trained on MNIST and in independently trained language-model SAEs, illustrating how effective interference captures interactions among SAE features that geometric measures alone do not.
\end{itemize}

\begin{figure}[t]
    \vspace{-10pt}
    \centering
    \includegraphics[width=\linewidth]{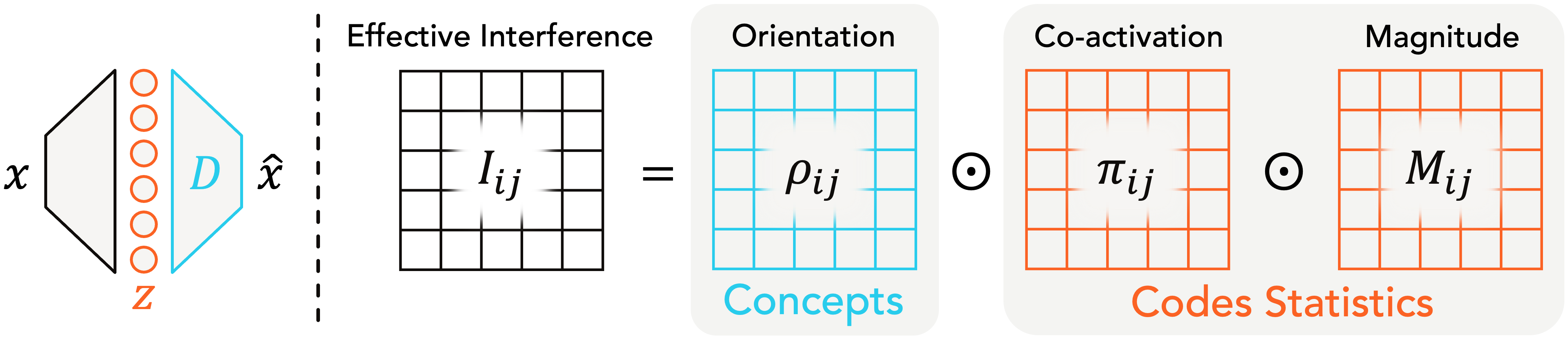}
    \vspace{-5pt}
    \caption{\textbf{Effective interference is a joint property of concept geometry and code statistics.} It factors as $I_{ij}=\rho_{ij}\pi_{ij}M_{ij}$, where $\rho_{ij}=\dvec_i^\top\dvec_j$ captures the geometric overlap between dictionary elements $\dvec_i,\dvec_j$, while $\pi_{ij}M_{ij}=\E[z_i z_j]$ captures how their corresponding codes $z_i,z_j$ are used across the data, through co-activation frequency $\pi_{ij}$ and conditional magnitude $M_{ij}$.}
    \label{fig:figinterference}
    \vspace{-10pt}
\end{figure}
%
%%%%%%%%%%%%%%%%%%%%%%%%%%%%%%%%%%
%%%%%%%%%%%%%%%%%%%%%%%%%%%%%%%%%%
%
\vspace{-1mm}
\section{Related Work}\label{app:related-work}
%

%%%%%%%%%%%%%%%%%%%%%%%%%%%%%%
\textbf{Measuring interference.}\quad Under superposition~\citep{elhage2022toymodelssuperposition}, concept directions overlap, yet interference is measured almost exclusively through geometry, depending only on the dictionary. This geometric view treats interference as a cost, motivating global overlap penalties~\citep{korznikov2025ortsae}, conditional orthogonalization of co-active concepts~\citep{costa2025from}, and geometric design and evaluation~\citep{lee2025evaluating}. Interference patterns also transfer across models and can be exploited by adversarial examples, linking them to network vulnerabilities~\citep{stevinson2026adversarial,gorton2025adversarial,gong2026signal}. Few works consider the accompanying code statistics. Exceptions formulate superposition information-theoretically~\citep{bereska2025superposition}, study concept co-occurrence~\citep{Li_2025}, use co-activation to identify causally relevant concepts~\citep{deng2026sparse}, or assess monosemanticity via latent coherence~\citep{filus2026measuring}. \citet{tolooshams2025sparse} show that Fourier-neural-operator SAEs admit broader geometric correlations because their functional codes specify where and how concepts are expressed rather than only their magnitude; in our framework, this may keep effective interference low despite geometric correlation. \citet{ivanov2026spectral} distinguish \emph{incidental} overlap between unrelated concepts~\citep{lecomte2023causes} from \emph{structural} interference reflecting relationships in the data distribution~\citep{kantamneni2025language,gurnee2024language}. Conditional orthogonalization~\citep{costa2025from} recognizes co-activation but modifies inference rather than measuring interference. We instead measure interference jointly via concept geometry and code statistics.

%%%%%%%%%%%%%%%%%%%%%%%%%%%%%%
\textbf{Superposition.}\quad Prior work studies superposition through generative models with prescribed sparse features and code distributions, asking how they are represented in lower-dimensional spaces~\citep{chowdhury2026effectssparsitysuperpositionloss}. Tied bottleneck autoencoders show that sparsity permits more features than hidden dimensions at the cost of interference~\citep{elhage2022toymodelssuperposition}, while subsequent theory characterizes the capacity of finite-dimensional representations to encode and compute with features~\citep{garg2026many,hanni2024mathematical,adler2024complexity,borobia2026linear}. \citet{prieto2026from} further show that co-activation statistics shape interference patterns to reflect data structure. We study the complementary inverse question: without positing a generative model, we start from model activations and ask what interference structure an extractor recovers.

%%%%%%%%%%%%%%%%%%%%%%%%%%%%%%
\textbf{Concept extraction and sparse autoencoders.}\quad Concept-based interpretability extracts directions from neural activations~\citep{doshi2017towards,kim2018interpretability}. ACE~\citep{ghorbani2019towards}, ICE~\citep{zhang2021invertible}, and CRAFT~\citep{fel2023craft} can be viewed as dictionary learning~\citep{mairal2009online,pmlr-v35-agarwal14a,fel2023holistic}, while newer methods extract features locally~\citep{hindupur2025projecting,shafran2026from}. Building on sparsity as an interpretability prior~\citep{mairal2014sparse,ribeiro2016should,lipton2018mythos}, SAEs likewise learn concept dictionaries with sparse codes. They are widely used in mechanistic interpretability~\citep{cunningham2023sparse,gao2025scaling,bussmann2024batchtopk}, with applications beyond language in vision and science~\citep{pach2025sparse,fel2025into,adams2025mechanistic,gujral2025sparse,simon2025interplm,nair2026interpreting}. This broad use has focused attention on how architecture and optimization shape recovered concepts. \citet{hindupur2025projecting} show that SAE projections encode distinct data assumptions, emphasizing alignment between the encoder and target structure. This has motivated architectures for temporal structure~\citep{lubana2026priors,bhalla2026temporal} and concept hierarchies~\citep{park2024the}, as well as objectives such as Matryoshka~\citep{bussmann2025learning} for feature absorption~\citep{chanin2024a}.

%%%%%%%%%%%%%%%%%%%%%%%%%%%%%%
\textbf{Theory of sparse autoencoders.}\quad Sparse autoencoders are closely related to sparse coding, or dictionary learning~\citep{olshausen1997sparse}. Classical recovery theory gives conditions under which gradient-based methods recover a ground-truth dictionary in shallow SAEs~\citep{pmlr-v35-agarwal14a,chatterji2017alternating,papyan2017working,arorasparse,nguyen2019dynamics} and deep unrolled architectures~\citep{tolooshams2022stable}, typically assuming independent sparse codes and dictionary incoherence~\citep{mutual_coherence}. Recent SAE analyses identify limits to concept recovery~\citep{klindt2025superposition}, including insufficient sparsity~\citep{cui2026on}, non-identifiability and spurious partial minima~\citep{tang2025theoretical}, and instability across runs yielding different dictionaries and codes~\citep{nelson2026toward,fel2025archetypal,paulo2026sparse}. Our closest connection is \citet{dorrell2026optimality}, who avoids generative ground truth and derive optimality constraints for $\ell_1$-regularized dictionary learning that explain splitting and absorption~\citep{chanin2024a}. We define interactions within a learned additive representation and study how an amortized encoder accommodates them across supports, yielding a frequency-magnitude interference profile and directional encoder-decoder cross-talk not captured by prior optimality constraints.
\vspace{-1mm}
\section{Effective Interference in Additive Representations}\label{sec:interference}
\vspace{-1mm}
\textbf{Notation.}\quad Lowercase $a$, bold lowercase $\bm a$, and uppercase $\A$ denote scalars, vectors, and matrices, respectively; $\bm a_i$ is the $i$-th column of $\A$. We use $\odot$ for the Hadamard product, $[a]_+:=\max\{a,0\}$, and $[a]_-:=\min\{a,0\}$. For a vector $\z$, $\operatorname{supp}(\z)=\{i:z_i\neq0\}$, and $\z$ is $k$-sparse if $|\operatorname{supp}(\z)|=k$. Given a support $S$, $\z_S$ and $\D_S$ restrict $\z$ to entries in $S$ and $\D$ to columns in $S$. We call $\x$ a \emph{model activation} and $\z(\x)$ its \emph{code}, and use ``directions'', ``features'' and ``concepts'' interchangeably.

Consider an additive representation 
\vspace*{-12pt}
\begin{equation}
\widehat{\x}=\D\z(\x)\ =\ \sum_{i=1}^{p}\dvec_i z_i(\x)
\vspace*{-4pt}
\end{equation}
where $\x\sim P$, $\D\in\R^{m\times p}$ contains $p$ nonzero concepts, and $\z(\x)\in\R^p$ their input-dependent codes. We distinguish \emph{potential} interactions, determined by dictionary geometry, from interactions realized by the codes across the data. We assume $\E\|\dvec_i z_i(\x)\|_2^2<\infty$ for all $i$ and, unless stated otherwise, nonnegative active codes and normalized concepts $\|\dvec_i\|_2=1$. Beyond these conditions, we assume nothing about how the codes are obtained, their sparsity, or the data-generating distribution $P$.
\begin{definition}[Effective interference]
\label{def}
For two distinct features $i\neq j$, define their sample-wise interaction as $\iota_{ij}(\x)
:=(\dvec_i^\top\dvec_j)z_i(\x)z_j(\x)$. Their effective interference is the expected interaction
\begin{equation}
I_{ij}
:=\E[\iota_{ij}(\x)]
=(\dvec_i^\top\dvec_j)\E[z_i(\x)z_j(\x)].
\label{eq}
\end{equation}
\end{definition}
Effective interference captures the realized, code-weighted interaction between learned features, and is symmetric (i.e., $I_{ij}=I_{ji}$). Importantly, it is an expected component cross-term within a chosen decomposition. These pairwise interactions can be collected into a matrix (\Cref{fig:figinterference}), whose off-diagonal entries measure pairwise feature interactions and whose diagonal entries capture individual feature contributions. This distinction follows from the reconstruction-energy decomposition
\begin{equation}
\E\|\widehat{\x}\|_2^2
=
\sum_{i=1}^{p}E_i+\sum_{i=1}^{p}\sum_{j\neq i}I_{ij}
=\sum_{i=1}^{p}E_i+2\sum_{i<j}I_{ij},
\end{equation}
where $E_i:=\E\|\dvec_i\z_i(\x)\|_2^2$ and $I_{ij}$ denote individual and pairwise contributions, respectively. Thus, $I_{ij}>0$ increases reconstruction energy relative to the individual contributions and is \emph{constructive}, whereas $I_{ij}<0$ is \emph{destructive}. These terms do not by themselves imply benefit or harm to reconstruction error or downstream behaviour. Moreover, this terminology describes how components combine within $\widehat{\x}$, not their effect on reconstruction error: the sign of $I_{ij}$ alone does not determine the change in $\E\|\widehat{\x}-\x\|_2^2=\E\|\widehat{\x}\|_2^2 -2\E[\x^\top\widehat{\x}] +\E\|\x\|_2^2$, since the alignment term $-2\E[\x^\top\widehat{\x}]$ also contributes.

As a joint geometric–statistical quantity, effective interference has three useful properties on support awareness, reparameterization invariance, and sign under nonnegative coding (see~\Cref{prop:interference-properties}). Within this formulation, geometry determines interaction orientation, whereas codes determine whether and how strongly it is realized. Hence, unlike cosine similarity, effective interference is support-aware and scale-aware, offered by the code statistics.

%%%%%%%%%%%%%%%%%%%%%%
\textbf{Interference profile.}\quad The second moment $\E[z_iz_j]$ combines how often two features co-activate with how large their components are when they do. To separate these effects, define
\begin{equation}
    a_i:=\mathbf{1}\{z_i\neq0\},
    \qquad\qquad
    \pi_{ij}:=\Pr(a_i=1,a_j=1),
    \qquad\qquad
    \rho_{ij}:= \dvec_i^\top\dvec_j,
    \label{eq:interference-factors}
\end{equation}
where $a_i$ is the activation indicator, $\pi_{ij}$ is the pair's co-activation frequency, and $\rho_{ij}$ is its geometric orientation. For $\pi_{ij}>0$, let $M_{ij}:=\E[z_iz_j\mid a_i=1,a_j=1]$. Hence, effective interference can be factorized. For every pair with $\pi_{ij}>0$, the effective interference can be factorized as follows:
\begin{equation}
   \boxed{
   I_{ij} =\rho_{ij}\ \pi_{ij}\ M_{ij}
   }
    \label{eq:interference-factorization}
\end{equation}
If $\pi_{ij}=0$, then $I_{ij}=0$ and the conditional magnitude $M_{ij}$ need not be defined. We call $\Pi_{ij}:=(\rho_{ij},\pi_{ij},M_{ij})$ the pair's \emph{effective-interference profile}. $I_{ij}$ measures its average contribution, while $\Pi_{ij}$ distinguishes geometric orientation, interaction exposure, and conditional magnitude. Consequently, equal values of $I_{ij}$ may arise from frequent weak co-activation or rare strong co-activation.

To compare interference across representations, we separate positive and negative pairwise contributions and define the normalized signed interferences
\begin{equation}
\mathcal{I}_+
:=
\frac{\E\!\left[2\sum_{i<j}[\iota_{ij}(\x)]_+\right]}
{\E[\|\widehat{\x}\|_2^2]},
\qquad
\mathcal{I}_-
:=
\frac{\E\!\left[2\sum_{i<j}[\iota_{ij}(\x)]_-\right]}
{\E[\|\widehat{\x}\|_2^2]},
\qquad
\mathcal{I}_{\mathrm{net}}
:=
\mathcal{I}_+ + \mathcal{I}_-.
\end{equation}
The signed terms distinguish constructive from destructive interference, while $\mathcal{I}_{\mathrm{net}}$ captures their net effect and may vanish through cancellation even when both are large.

 Overall, the interference profile distinguishes several regimes that a single geometric or aggregate statistic conflates. Large $|\rho_{ij}|$ together with $\pi_{ij}\approx0$ describes \emph{unrealized overlap}: the directions admit interaction, but their supports prevent it. Large $\pi_{ij}$ and small $M_{ij}$ describe frequent weak interactions, whereas small $\pi_{ij}$ and large $M_{ij}$ describe rare strong interactions.

%
%%%%%%%%%%%%%%%%%%%%%%%%%%%%%%%%%%
%%%%%%%%%%%%%%%%%%%%%%%%%%%%%%%%%%
%
% \vspace*{-3em}
\section{How Sparse Autoencoders Handle Interference}
\label{sec:local optimality}
The effective-interference profile characterizes how features interact in a learned representation. The architecture and optimization of the extractor determine which profiles arise. Consider a typical amortized SAE architecture~\citep{cunningham2023sparse,bricken2023monosemanticity,gao2025scaling},
\begin{equation}
\z=\sigma(\W^\top\x+\bvec),
\qquad
\widehat{\x}=\D\z,
\label{eq:sae}
\end{equation}
where $\W$ and $\D$ are the encoder and decoder weights, $\bvec$ is the encoder bias, and $\sigma$ is a sparsifying nonlinearity. Unlike normalized concepts in $\D$, encoder weights $\W$ need not be normalized. An SAE shapes interference at two levels: encoder scores together with $\sigma$ determine the support and hence $\pi_{ij}$, while the encoder and decoder govern interactions among co-active features and hence $\rho_{ij}$ and $M_{ij}$. We study the latter on a fixed support, so our conclusions are local to an active region and, for gradient dynamics, to updates that preserve the support. Within a fixed support $S$, the encoder is locally affine:
\vspace{-5pt}
\begin{equation}
\z_S=\W_S^\top\x+\bvec_S,
\qquad
\z_{S^c}=0.
\label{eq:local-affine-encoder}
\vspace{1pt}
\end{equation}
Rather than optimizing codes separately for each input, an SAE predicts them with a shared encoder. The map in~\eqref{eq:local-affine-encoder} thus amortizes fixed-support least squares across inputs. \Cref{prop:fixed-support-optimality} sets code-optimality condition for each input example. We treat this condition as an ideal target, and develops this section's analyses for when the code-optimality gap is small. This code-optimality gap relates to amortization gap, and can be small for in-distribution data~\citep{margossian2023amortized} (\Cref{sec:results} further investigates this in SAEs). 
\begin{restatable}[Code-optimality condition]{proposition}{fixedsuppoptimality}
\label{prop:fixed-support-optimality}
For a fixed decoder $\D$ and support $S$, any interior least-squares optimum satisfies
\begin{equation}
    \D_S^\top \res = 0,
    \qquad
    \res := \x - \D_S \z_S.
\label{eq:fixed-support-opt}
\end{equation}
For nonnegative codes, interior means $z_i>0\ \forall i\in S$.
\end{restatable}
Code optimality is a decoder condition: it constrains $\D$ given the codes, but not the encoder $\W$. We therefore require an encoder consistency identity (\Cref{prop:enc-consistency-cond}). Let $\alpha_i=\|\w_i\|_2$ and $\gamma_{ij}:=\w_i^\top\dvec_j/\|\w_i\|_2$ be the cosine similarity between $\w_i$ and $\dvec_j$, so that $\w_i^\top\dvec_j=\alpha_i\gamma_{ij}$. Unlike $\D^\top\D$, which captures interactions among reconstruction components, $\W^\top\D$ couples encoder and decoder directions and is asymmetric.
\begin{restatable}[Encoder consistency identity]{proposition}{encconsidn}
\label{prop:enc-consistency-cond}
Suppose the encoder is locally affine as in~\eqref{eq:local-affine-encoder}. For every active feature $i\in S$,
\begin{equation}
    0=(\alpha_i\gamma_{ii}-1)z_i+
    \sum_{j\neq i}\alpha_i\gamma_{ij}z_j+b_i+\w_i^\top\res.
    \label{eq:sample-balance}
\end{equation}
Consequently, for every $i$ with $\Pr(a_i=1)>0$,
\begin{align}
0=\underbrace{(\alpha_i\gamma_{ii}-1)\E[z_i\mid a_i=1]}_{\text{self-response mismatch}}\
    +\ \underbrace{\sum_{j\neq i}\alpha_i\gamma_{ij}
    \E[z_j\mid a_i=1]}_{\text{incoming cross-talk}}
+\underbrace{b_i}_{\substack{\text{bias}\\\text{compensation}}}
    +\underbrace{\E[\w_i^\top\res\mid a_i=1]}_{\text{residual coupling}},
    \label{eq:conditional-balance}
\end{align}
\end{restatable}
\vspace{-3mm}
where $\gamma_{ii}$ and $\gamma_{ij}$ are encoder-decoder self-alignment and cross-alignment, respectively. The encoder gain $\alpha_i$ rescales the entire response of encoder $i$, including its self-response, incoming cross-talk, and residual coupling. With the code-optimality condition, \Cref{prop:enc-consistency-cond} exposes which architectural terms can accommodate interactions. Its incoming cross-talk is directional: the effect of feature $j$ on the encoder score of feature $i$ depends on $\w_i^\top\dvec_j$. This contrasts with the symmetric decoder interference, $I_{ij}=(\dvec_i^\top\dvec_j)\E[z_i z_j]$. Moreover, for tied weights, \eqref{eq:fixed-support-opt} removes the residual coupling term in~\eqref{eq:conditional-balance}, but not generally for an untied encoder, since $\D_S^\top\res=0$ does not imply $\W_S^\top\res=0$.

To isolate the role of each term in the encoder-consistency balance, we begin with a tied, unit-norm, bias-free baseline and relax one constraint at a time by introducing bias, encoder gain, and independent encoder directions. Holding the remaining degrees of freedom fixed allows each change in interference to be attributed to a specific accommodation mechanism. These settings serve as analytical controls rather than assumptions about practical SAEs.

\textbf{Route I: constrained local orthogonalization.}\quad
Consider the tied and normalized $\alpha_i=1$, $\gamma_{ii}=1$, $\gamma_{ij}=\rho_{ij}$, and bias-free $\bvec=0$ setting. Under the fixed-support code optimality condition per input~\eqref{eq:fixed-support-opt}, the encoder-consistency condition~\eqref{eq:sample-balance} gives
\begin{equation}
    (\D_S^\top\D_S-\eye)\z_S=0.
\end{equation}
For a single input, this only constrains $\D^{\top}_S\D_S-\eye$ along the observed code vector $\z_S$. However, if the code vectors observed on support $S$ span $\R^{|S|}$, the condition must hold in every direction and therefore implies
$\D^\top_S\D_S=\eye$. Thus, tied normalized SAEs are driven toward orthogonality among concepts that co-activate, rather than across the entire overcomplete dictionary.

The same local orthogonalizing tendency can be seen directly from the training dynamics. For two active features $S=\{i,j\}$, define $\Delta(\D_S):=\|\D_S^\top\D_S-\eye\|_F^2 =2\rho_{ij}^2$, where the second equality follows from unit-norm columns.
\newpage

\vspace*{-20pt}

\begin{restatable}[Local orthogonalization for two active features]{proposition}{localorth}
\label{prop:local-orthogonalization}
Consider the tied, normalized, bias-free model with fixed support $S=\{i,j\}$ and loss $\frac12\|\widehat{\x}-\x\|_2^2$. Let $\D_S^+$ denote the active dictionary after one gradient step of size $\eta$, followed by column normalization. Then
\begin{equation}
    \Delta(\D_S^+)-\Delta(\D_S)
    =-8\eta\rho_{ij}^2
      \|z_i\dvec_i-z_j\dvec_j\|_2^2+O(\eta^2).
    \label{eq:local-orthogonalization-update}
\end{equation}
Hence, if $\rho_{ij}\neq0$ and $z_i\dvec_i\neq z_j\dvec_j$, a gradient step cannot increase $\Delta(\D_S)$ to first order.
\end{restatable}
\textcolor{myred}{Equation}~\eqref{eq:local-orthogonalization-update} identifies an orthogonalizing contribution on examples where the pair is jointly active. Thus, $\pi_{ij}$ controls the frequency of this direct two-feature contribution, although other support states can also change the pairwise angle. The dynamics therefore predict selective orthogonalization of frequently co-active directions, not global dictionary orthogonalization. \Cref{prop:local-orthogonalization} concerns one-example gradient descent with a support-preserving step. \Cref{sec:results} further shows experiments with higher $k$, batch updates, and Adam optimizer where this predicted tendency still remains.

\textbf{Route II: conditional bias compensation.}\quad We next retain tied, unit-norm encoder-decoder weights but allow a learned encoder bias. Tying implies $\alpha_i=1$, $\gamma_{ii}=1$, and $\gamma_{ij}=\rho_{ij}$ for $i\neq j$. Hence, the self-response mismatch vanishes, while fixed-support code optimality removes residual coupling. We obtain the following population-level condition:
\begin{equation}
    \hat{b}_i = -\sum_{j\neq i} \rho_{ij}\E[z_j\mid a_i=1].
    \label{eq:bias-prediction}
    \vspace{-3pt}
\end{equation}
Thus, a shared bias can compensate for the \emph{mean incoming cross-talk} on examples where feature $i$ is active. It cannot generally cancel the sample-wise cross-talk separately for every input.

\textbf{Route III: encoder-gain adaptation.}\quad We next keep the encoder direction tied to its corresponding decoder direction, but allow the encoder norm to vary:
$\w_i=\alpha_i\dvec_i$, with $\alpha_i>0$. Consequently, $\gamma_{ii}=1$ and $\gamma_{ij}=\rho_{ij}$. In this setting, fixed-support code optimality also gives $\w_i^\top\res=0$. The conditional balance (\Cref{prop:enc-consistency-cond}) therefore gives the following estimate on the encoder gain
\begin{equation}
    \hat{\alpha}_i = \left(1 +\ \frac{\sum_{j\neq i}\rho_{ij} \E[z_j\mid a_i=1]}{\E[z_i\mid a_i=1]}\right)^{-1}.
    \label{eq:norm-balance}
\end{equation}
Thus, positive mean incoming cross-talk predicts $\alpha_i<1$, whereas negative incoming cross-talk predicts $\alpha_i>1$, provided the denominator is positive. Encoder gain rescales both the self-response and incoming cross-talk, providing an alternative to decoder angular orthogonalization. This behaviour manifests itself in the unnormalized learning dynamics (\Cref{prop:local-orthogonalization-unnormalized}): the two-feature angular dynamics depend on both the pair's cosine similarity and its weight norm.

\textbf{Route IV: encoder-decoder duality.}\quad Untying separates decoder interaction from inference cross-talk. For each pair, $(\rho_{ij},\pi_{ij}M_{ij})$ describes decoder interaction, while
$\w_i^\top\dvec_j=\alpha_i\gamma_{ij}$ is the directed cross-gain through which decoder feature $j$ enters encoder score $i$. Fixed-support least squares admits a biorthogonal encoder~\eqref{eq:support-dual}, with targets
$\alpha_i\gamma_{ii}=1$ and $\alpha_i\gamma_{ij}=0$ for $j\neq i$. Because different supports induce different duals, a shared encoder can only approximate these targets jointly. We therefore predict small $|\alpha_i\gamma_{ij}|$ for pairs with large $\pi_{ij}M_{ij}$ even when $\rho_{ij}>0$, particularly when residual coupling is small. When $\alpha_i$ is bounded away from zero, this also predicts $\gamma_{ij}$ near zero. Thus, untying can preserve constructive effective interference while suppressing incoming cross-talk, reflecting the roles of encoders as feature detectors and decoders as reconstruction components~\citep{bricken2024dictionary,rajamanoharan2024improving, Nanda2023}.

The four routes are controlled limits of the same balance and are not mutually exclusive. A practical encoder may distribute accommodation across decoder geometry, encoder gain, encoder-decoder alignment, bias, and residual coupling. We now show how these signatures recur across SAE families.

\begin{figure}[t]
    \vspace{-5mm}
    \centering
    \includegraphics[width=\linewidth]{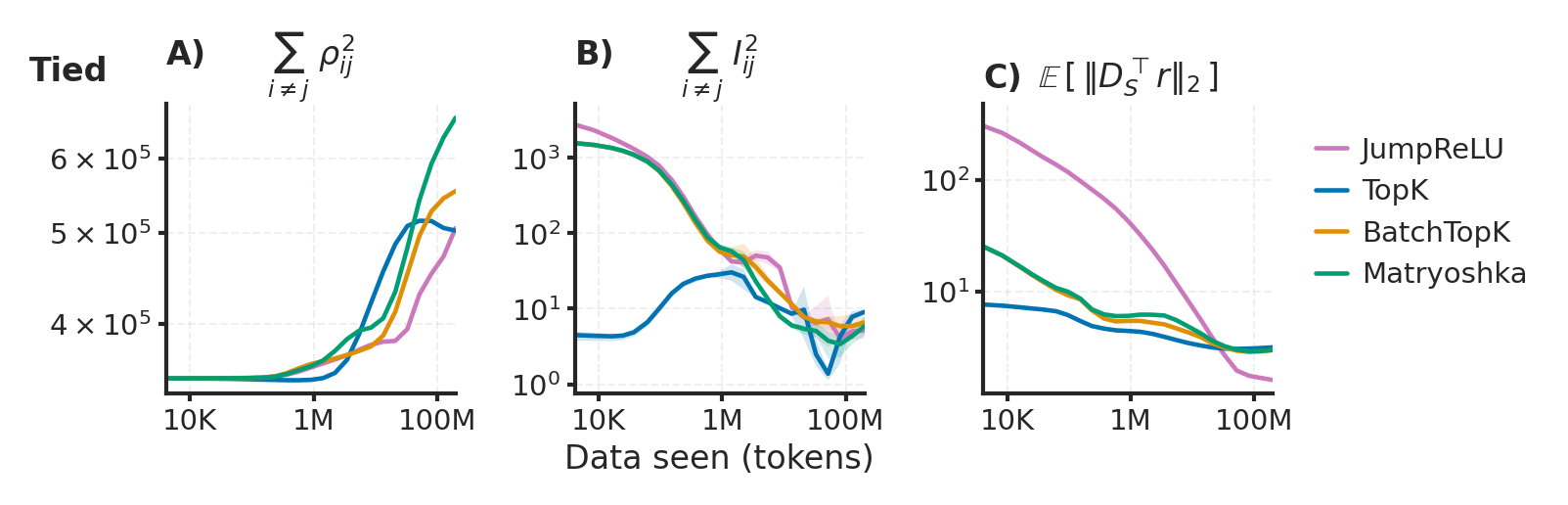}
    \vspace{-8mm}
\caption{\textbf{Geometric and effective interference have distinct training dynamics.} \textbf{A)} Global geometric overlap rises, whereas \textbf{B)} effective interference falls from its peak. Consistent with~\Cref{prop:local-orthogonalization}, this reflects selective orthogonalization of co-active pairs, reducing realized interactions without globally orthogonalizing the dictionary. \textbf{C)} Residual coupling declines but remains nonzero, so fixed-support code optimality (\Cref{prop:fixed-support-optimality}) is approached only approximately (see~\Cref{fig:local-orthogonalization-flame} for joint distribution of decoder orientation $\rho_{ij}$ and co-activation-weighted magnitude $\pi_{ij}M_{ij}$).}
    \label{fig:local-orthogonalization}
    \vspace{-1pt}
\end{figure}
%
%%%%%%%%%%%%%%%%%%%%%%%%%%%%%%%%%%
%%%%%%%%%%%%%%%%%%%%%%%%%%%%%%%%%%
%
\vspace{-1mm}
\section{Results}\label{sec:results}
\vspace{-1mm}
We evaluate whether the architectural degrees of freedom identified in~\Cref{sec:local optimality} predict the interference structure learned by SAEs. Our experiments use TopK~\citep{gao2025scaling,makhzani2013k}, JumpReLU~\citep{rajamanoharan2024jumping}, BatchTopK~\citep{bussmann2024batchtopk}, and Matryoshka~\citep{bussmann2025learning} SAEs trained on Pythia-160M-deduped residual-stream activations at layer 8 with dictionary width $p = 16{,}384$ and sparsity $k = 40$. For additional results and experimental details, see \Cref{app:additional results,app:exp-details}, respectively.

\textbf{How to interpret the results.}\quad We do not rank SAE families by absolute interference. Instead, each family tests the same mechanisms by measuring how untying, encoder bias, and gain adaptation alter interference and encoder-decoder cross-talk. The evidence lies in the direction and recurrence of these within-family changes, not absolute differences across sparsifiers.

We employ the code statistics on activation frequency and magnitude, $\pi_{ij}M_{ij}$, together with the orientation entry $\rho_{ij}$ to show how pairwise interaction is allocated between the decoder and codes. For when encoder gain is enabled, we report $\alpha_i$ and encoder-decoder cross-alignment $\gamma_{ij}$.

\begin{figure}[t]
    \centering
    \includegraphics[width=\linewidth]{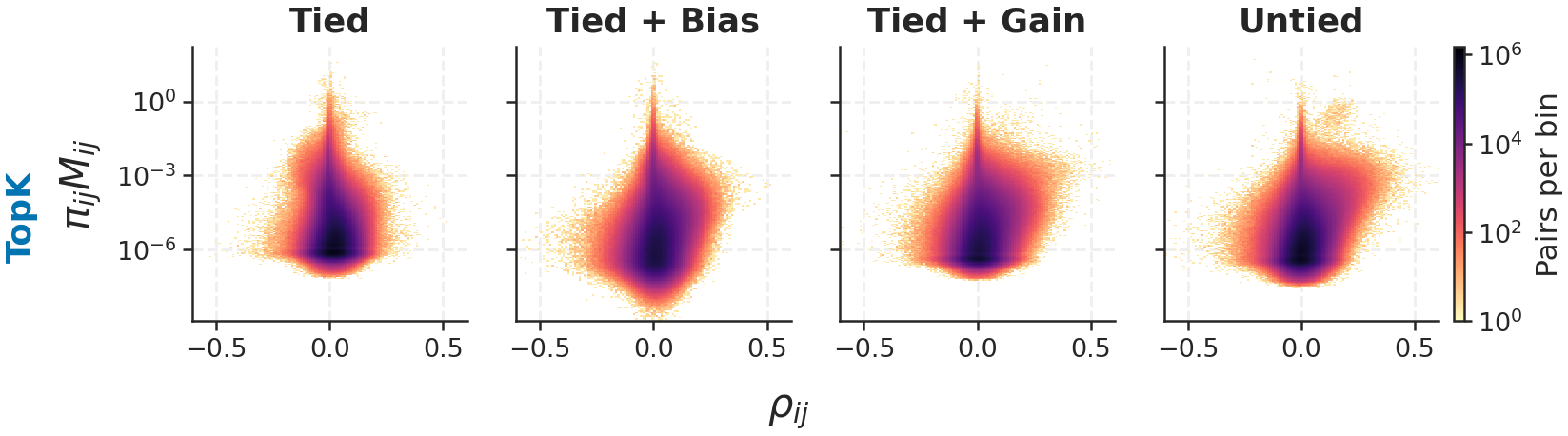}
    \vspace{-15pt}
    \caption{\textbf{Architectural freedom shifts co-active feature pairs toward constructive decoder interactions.} Joint distributions of orientation $\rho_{ij}$ and co-activation-weighted magnitude $\pi_{ij}M_{ij}$. In the tied, normalized, bias-free model, pairs with large $\pi_{ij}M_{ij}$ concentrate near $\rho_{ij}=0$, consistent with local orthogonalizing pressure. With bias, encoder-gain freedom, or an untied encoder, some of these pairs extend toward positive $\rho_{ij}$. This pattern recurs across other SAEs (\Cref{fig:architectural-regimes}).}
    \label{fig:topk-architectural-regimes}
    \vspace{-12pt}
\end{figure}
\textbf{Constrained SAEs suppress realized interactions without globally orthogonalizing.}\quad
\Cref{prop:local-orthogonalization} predicts support-local rather than global orthogonalizing pressure. Across SAE families, global squared geometric overlap increases during training, whereas squared effective interference falls from its peak (\Cref{fig:local-orthogonalization}\textcolor{myred}{A-B}). Thus, the dictionary need not become globally orthogonal for realized interactions to decrease; code statistics determine which overlaps matter. This behaviour is consistent with selective pressure on co-active pairs and refines the global orthogonality picture suggested by the linear representation hypothesis~\citep{costa2025from}. The code-optimality gap $\E[\|\D_S^\top\res\|_2]$ also decreases but remains nonzero (\Cref{fig:local-orthogonalization}\textcolor{myred}{C}), showing that fixed-support optimality is an increasingly accurate, but still approximate, description of the learned codes. This behaviour is aligned with prior investigations of amortization gap in SAEs~\citep{o2024compute}.

\textbf{Architectural freedom enables constructive decoder interactions.}\quad The joint interference profiles in~\Cref{fig:topk-architectural-regimes} reveal a common effect of relaxing architectural constraints. In the tied, normalized, bias-free model, pairs with large $\pi_{ij}M_{ij}$ concentrate near $\rho_{ij}=0$ or shift toward negative. Adding bias or encoder gain, or untying the encoder, allows such pairs to extend toward positive $\rho_{ij}$ and therefore retain constructive effective interference. These profiles establish the common decoder behaviour but do not identify how each architecture accommodates the resulting interactions; the encoder-consistency balance \eqref{eq:conditional-balance} provides this distinction.

\newpage

\begin{wrapfigure}[58]{r}{0.44\textwidth}
\vspace{-7mm}

\begin{minipage}{\linewidth}
  \centering
  \includegraphics[width=\linewidth]{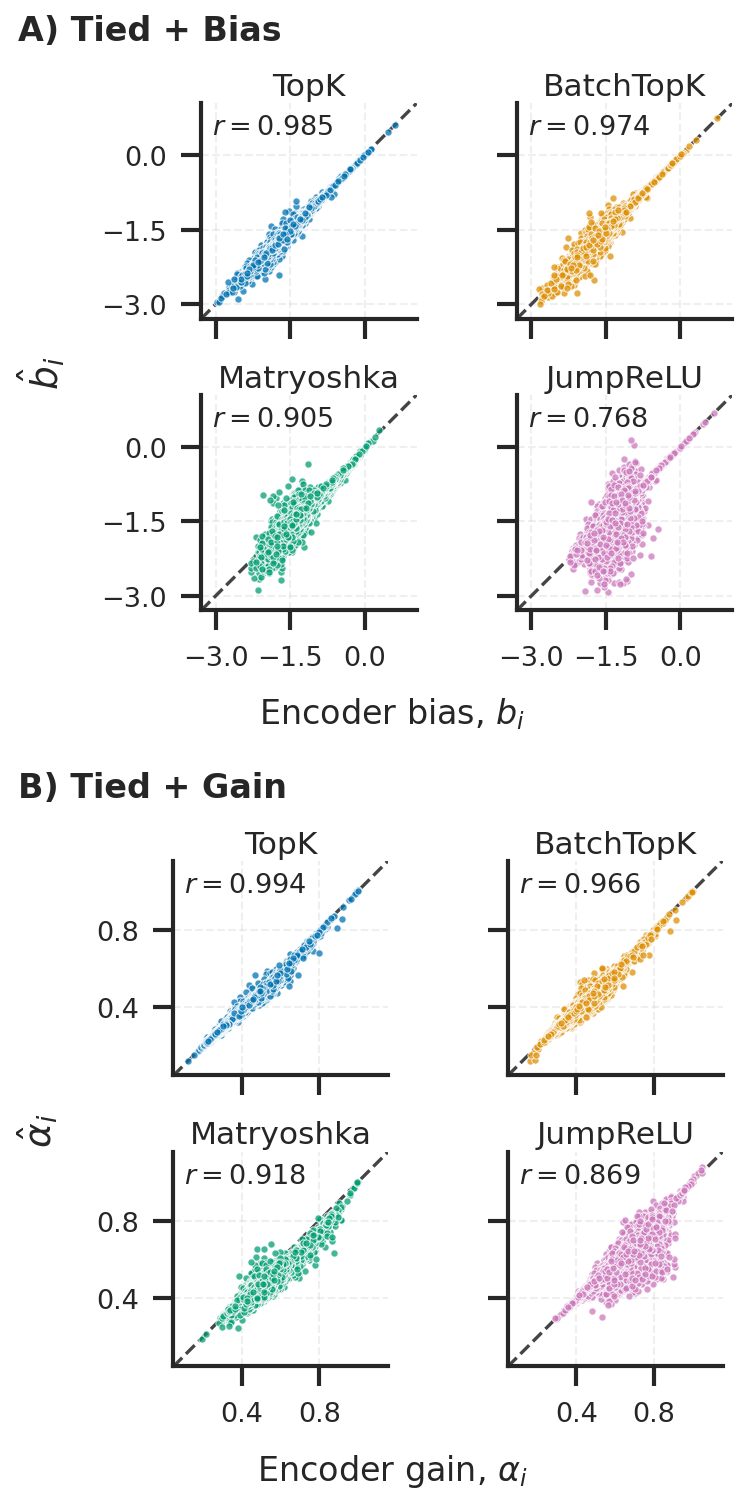}
  \vspace{-7mm}
  \caption{\textbf{Bias and encoder gain compensate mean incoming cross-talk.}
  Each point is a feature; horizontal axes show learned parameters and vertical axes their encoder-balance predictions.
  \textbf{A)} In tied + bias models, $\hat b_i$ from~\eqref{eq:bias-prediction} closely matches $b_i$.
  \textbf{B)} In tied + gain models, $\hat\alpha_i$ from~\eqref{eq:norm-balance} tracks $\alpha_i$.
  Dashed lines indicate equality and $r$ is Pearson correlation.
  This behaviour appears across other SAEs (\Cref{fig:bias-alpha-more}).}
  \label{fig:bias-alpha}
\end{minipage}

\vspace{2mm}

\begin{minipage}{\linewidth}
  \centering
  \includegraphics[width=\linewidth]{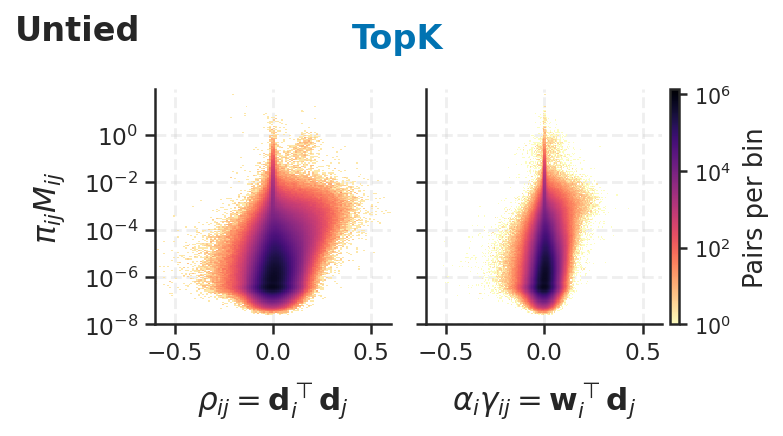}
  \vspace{-20pt}
  \caption{\textbf{Untying separates decoder interactions from encoder cross-talk.}
  Joint distributions of co-activation-weighted magnitude $\pi_{ij}M_{ij}$ against decoder orientation $\rho_{ij}$ and cross-gain $\alpha_i\gamma_{ij}=\w_i^\top\dvec_j$.
  Pairs with large $\pi_{ij}M_{ij}$ tilt toward positive decoder alignment, while their cross-gains concentrate near zero.}
  \label{fig:encdec}
\end{minipage}

\vspace{-6mm}
\end{wrapfigure}

\vspace*{-20pt}

With bias, the predicted $\hat b_i$ closely matches the learned $b_i$ across SAE families (\Cref{fig:bias-alpha}\textcolor{myred}{A}), supporting compensation of mean incoming cross-talk. TopK and BatchTopK lie closest to equality, whereas Matryoshka and JumpReLU deviate more and correlate less. Because $\hat b_i-b_i$ equals the residual-coupling term in this controlled setting, these deviations indicate a less exact fixed-support approximation: its mean magnitude is $0.11-0.14$ for Matryoshka and JumpReLU against $0.04-0.05$ for TopK and BatchTopK. With encoder gain, $\hat\alpha_i$ similarly tracks the learned $\alpha_i$ (\Cref{fig:bias-alpha}\textcolor{myred}{B}), with the same ordering ($0.11-0.15$ versus $0.03-0.06$). Positive mean incoming cross-talk predicts $\alpha_i<1$, allowing the encoder to attenuate its self-response and incoming cross-talk without changing decoder normalization. This mechanism parallels \citet{prieto2026from}, who observe increased constructive interference on toy models trained with weight decay. Similarly, by training an unnormalized decoder, a smaller decoder norms relax the pressure toward local orthogonality (\Cref{prop:local-orthogonalization-unnormalized}), allowing positively aligned directions to constructively reconstruct.

Untying shifts decoder pairs toward positive alignment $\rho_{ij}$ and increases positive effective interference $I_{ij}$ (\Cref{fig:encdec} left), showing that decoder overlap need not be reduced to control inference cross-talk. The pairwise profiles reveal the mechanism: pairs with large $\pi_{ij}M_{ij}$ exhibit a pronounced tilt toward positive $\rho_{ij}$, while their encoder-decoder cross-gains $\alpha_i\gamma_{ij}=\w_i^\top\dvec_j$ remain thin and concentrated near zero (\Cref{fig:encdec} right). Untying therefore separates reconstruction from inference: the decoder retains constructive interactions, while the encoder limits how decoded feature $j$ enters the score of feature $i$. This pattern is consistent with approximate support-conditional biorthogonalization, although these pairwise distributions do not establish exact duality on every support.

\textbf{Architectural freedom reallocates interference handling.}\quad We decompose the normalized off-diagonal contribution into constructive $\mathcal I_+$ and destructive $\mathcal I_-$ terms, with $\mathcal I_{\mathrm{net}}=\mathcal I_+ + \mathcal I_-$ (\Cref{fig:aggregate-interference}). This separation is essential: in the tied baseline, sizable positive and negative contributions largely cancel, particularly for JumpReLU, so a small net value does not imply weak feature interactions. Within every SAE family, introducing bias, encoder gain, or independent encoder directions shifts the balance toward positive net interference relative to the tied baseline, although the intermediate variants are not strictly ordered because the accommodation mechanisms interact rather than combine additively. Untying produces the largest constructive contribution across all four families: $\mathcal I_+ + |\mathcal I_-|$ accounts for roughly $40$--$50\%$ of reconstruction energy, while the net contribution remains approximately $15$--$25\%$. Effective interference is therefore not a small residual effect, nor something SAEs generally minimize; architectural constraints determine whether cross-feature interactions are canceled or retained constructively. This refines a purely geometric view of the LRH~\citep{park2024the}: interference depends not only on dictionary overlap, but also on which features are jointly used.
Finally, we demonstrate the generality of these interference signatures beyond our main setting using raw MNIST images (\Cref{app:mnist}) and pre-trained language-model SAEs (\Cref{app:saebench}).

\textbf{Illustrating the interference profile across support size $k$ and dictionary width $p$.}\quad As a diagnostic example, we vary $k$ and $p$ in TopK SAEs (\Cref{fig:topk-kp-all}). Let $\tilde p$ denote the number of utilized features, accounting for features that never activate. The mean co-activation frequency $\overline{\pi}_{ij}$ closely follows the reference $\nicefrac{k(k-1)}{\tilde p(\tilde p-1)}$ (\Cref{app:barpi}). At fixed $p$, increasing $k$ raises $\overline{\pi}_{ij}$ while reducing the mean conditional magnitude $\overline{M}_{ij}$. Mean absolute decoder overlap $\overline{|\rho_{ij}|}$ also generally decreases in the biased, gain-adaptive, and untied variants. Net interference $\mathcal I_{\mathrm{net}}$, however, has no universal monotonic trend, indicating that increased co-activation can be offset by weaker magnitudes and signed geometric contributions. At fixed $k$, increasing $p$ spreads co-activation across more pairs, reducing $\overline{\pi}_{ij}$, while $\overline{M}_{ij}$ and $\overline{|\rho_{ij}|}$ change less in most settings. The largest tied configurations also show elevated FVU and are therefore interpreted cautiously. This illustrative experiment shows how the interference profile separates compensating statistical and geometric changes without asserting a universal scaling law.

\begin{figure}[t]
    \vspace{-5mm}
    \centering
    \includegraphics[width=\linewidth]{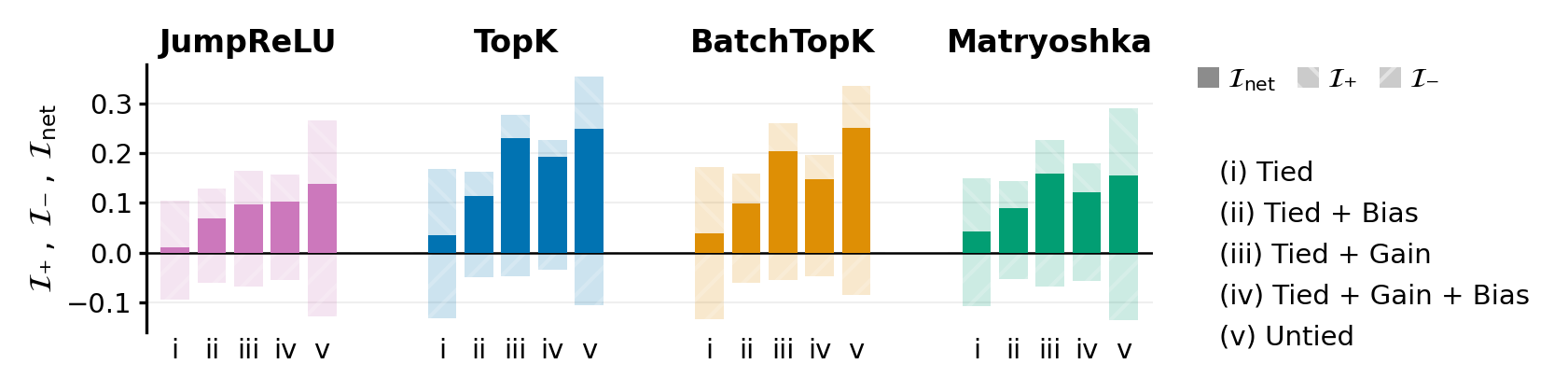}
    \vspace{-6mm}
    \caption{\textbf{SAE reconstruction contains substantial signed effective interference.} Constructive $\mathcal I_+$, destructive $\mathcal I_-$, and net $\mathcal I_{\mathrm{net}}$ contributions across architectural variants and SAE families. Sizable positive and negative terms can cancel, so small net interference does not imply weak interactions. Relative to the tied case, architectural freedom generally shifts the balance toward higher interference.}
     \vspace{-5mm}
    \label{fig:aggregate-interference}
\end{figure}
%
%%%%%%%%%%%%%%%%%%%%%%%%%%%%%%%%%%
%%%%%%%%%%%%%%%%%%%%%%%%%%%%%%%%%%
%
\vspace{-1mm}
\section{Conclusion}\label{sec:conclusion}
\vspace{-1mm}

\textbf{Discussion and limitations.}\quad Effective interference measures realized pairwise cross-contributions within a learned additive decomposition, allowing us to study how SAE architectures allocate reconstruction energy and which interaction regimes they favor. It does not characterize interactions among latent ground-truth concepts, as studied by \cite{prieto2026from}, nor is it invariant to phenomena such as splitting or absorption; identifying these requires evidence beyond $I_{ij}$ and aligned decompositions across extractors. More generally, effective interference diagnoses energy allocation, not representation quality: constructive and destructive describe cross-term signs without implying reconstruction quality, semantic interpretability, causal relevance, or downstream behaviour.

Our theory is local to fixed supports and approximates a shared amortized encoder by per-input code optimality. This approximation can be accurate in-distribution~\citep{margossian2023amortized} but degrade under distribution shift~\citep{o2024compute}, consistent with our small but nonzero measured gaps. Similarly, our local orthogonalization proposition considers two active features and a small, support-preserving gradient step, and thus does not establish global convergence or cover larger supports, minibatch training, adaptive optimizers, or support transitions. Overall, our experiments support the proposed compensation and orthogonalization mechanisms without uniquely identifying a single mechanism.

\textbf{Conclusion.}\quad Interference is not a property of feature geometry alone, but of geometry together with how features co-activate. We formalized this interaction as \emph{effective interference} and showed that SAE architectures realize distinct interference regimes: constrained encoders tend to orthogonalize co-active features, whereas bias, gain, and encoder freedom allow constructive cross-contributions to remain. These results provide a framework for characterizing how architectural choices shape the interactions realized by learned representations, without treating any single interference regime as universally preferable.

%%%%%%%%%%%%%%%%%%%%%%%%%%%%%%%%%%
%%%%%%%%%%%%%%%%%%%%%%%%%%%%%%%%%%

\subsubsection*{Acknowledgments}
V.C. and B.T. would like to thank Demba Ba, Sohini Gupta, and members of NeuBahar Lab for helpful discussions and feedback. V.C. and B.T. acknowledge funding from the Canada CIFAR AI Chairs Program and Alberta Machine Intelligence Institute (Amii). B.T. acknowledges support of Natural Sciences and Engineering Research Council of Canada (NSERC), RGPIN-2026-05959.

%%%%%%%%%%%%%%%%%%%%%%%%%%%%%%%%%%
%%%%%%%%%%%%%%%%%%%%%%%%%%%%%%%%%%
%
% \newpage

\bibliographystyle{iclr2027_conference}
\bibliography{references}

%%%%%%%%%%%%%%%%%%%%%%%%%%%%%%%%%%
%%%%%%%%%%%%%%%%%%%%%%%%%%%%%%%%%%
%
\appendix

\section{Appendix - AI Use Statement}\label{app:ai-use}

Authors used generative AI tools to assist with polishing and improving the writing of the manuscript, including clarity and readability.

We used Claude Code to adapt our existing SAE implementation, training procedure, and evaluation metrics, originally developed and validated on MNIST, to run at scale on cached Pythia activations. The underlying architectures, training objective, and evaluation quantities were fixed beforehand.

We also used GPT-5.6 to obtain feedback on the manuscript, including suggestions on the organization of the results and overall write-up.
We did not use generative AI tools for research formulation, experimental or methodological design, derivations, or theoretical development. We reviewed all AI-assisted work and take responsibility for the final content of this work, including text, claims, and artifacts produced with the aid of generative AI.

\section{Appendix - Reproducibility Statement}
\label{app:reproducibility}

Appendix~\ref{app:exp-details} provides the complete experimental setup, including the model and data, training budget, architectures, initialization, optimization, and constraints. It also reports holdout reconstruction, sparsity, and dead-feature statistics for all runs (\Cref{tab:eval}). Complete proofs of all formal results are provided in the appendix.

The SAEBench experiments (\Cref{app:saebench}) use publicly released checkpoints and SAEBench's own evaluation pipeline. We identify the checkpoints by model, layer, and sparsity target, allowing these results to be reproduced directly from public artifacts.

\section{Appendix - Ethics Statement}

Interpretability is an important component of developing and auditing reliable AI systems, as it can help researchers understand model representations, identify failure modes, and evaluate unintended behaviour. Sparse autoencoders are used for this purpose by extracting human-interpretable features from neural network activations. Our work highlights a consideration for this approach: the extent to which SAE features are separable depends on architectural choices that may otherwise appear incidental. We therefore view our results as supporting more cautious interpretation of individual SAE features and interventions based on them.

Our experiments use publicly released pretrained models, datasets, and SAE checkpoints. No new data were collected and no human subjects were involved. The language corpora are web-derived and may contain personal or biased content, but our analysis is restricted to aggregate geometric statistics and does not inspect or attempt to recover individual examples or sensitive attributes. As with interpretability research more broadly, the methods studied here could in principle contribute to model analysis or reverse engineering; however, our work does not introduce new capabilities for eliciting model behaviour or recovering training data, and we do not identify additional direct risks beyond those associated with existing interpretability methods.

\newpage
\section{Appendix - Effective Interference}\label{app:effective-interference}

The effective interference has the following properties.

\begin{proposition}[Properties of effective interference]
\label{prop:interference-properties}
For distinct features $i\neq j$:
\begin{enumerate}[leftmargin=8mm]
    \item \textbf{Support awareness.} If $\Pr(z_i\neq0,z_j\neq0)=0$, then $I_{ij}=0$, irrespective of $\dvec_i^\top\dvec_j$.
    \item \textbf{Reparameterization invariance.} For positive scalars $\beta_i$, the transformations $\dvec_i\mapsto\beta_i\dvec_i$ and $z_i\mapsto z_i/\beta_i$ leave every component $\dvec_iz_i$ and hence every $I_{ij}$ unchanged.
    \item \textbf{Sign under nonnegative coding.} If $z_i,z_j\geq0$ and the pair co-activates with nonzero probability, then $I_{ij}$ has the sign of $\dvec_i^\top\dvec_j$.
\end{enumerate}
\end{proposition} 

%%%%%%%%%%%%%%%%%%%%%%
\subsection{Relation to geometric measures}
\label{sec:geometry-special-case}
Effective interference recovers geometric measures only under restrictive code statistics. Suppose that $\E[z_iz_j]=s$ for every $i\neq j$. Then
\begin{equation}
    I_{ij}=s(\dvec_i^\top\dvec_j),
    \label{eq:gram-special-case}
\end{equation}
so the off-diagonal effective-interference matrix is proportional to the off-diagonal Gram matrix. Purely geometric interference therefore implicitly assumes that the code second moments contain no pair-specific information. In the binary, unit-normalized case, $M_{ij}=1$, and effective interference instead reduces to cosine similarity weighted by the probability of co-activation. This makes explicit the role of latent statistics in models of superposition~\citep{elhage2022toymodelssuperposition,prieto2026from}.

Mutual coherence plays a different role. For normalized directions, let
\begin{equation}
    \mu:=\max_{k\neq \ell}|\dvec_k^\top\dvec_\ell|.
\end{equation}
Under nonnegative coding,
\begin{equation}
    |I_{ij}|
    \leq \mu\,\pi_{ij}M_{ij}.
    \label{eq:coherence-bound}
\end{equation}
Coherence therefore gives a support-agnostic worst-case geometric bound, rather than a measure of the interaction realized by a particular representation~\citep{mutual_coherence,tropp2004greed}. The bound is loose when co-activation concentrates on pairs whose geometric overlap is far below the global maximum, mirroring the distinction between worst-case geometric recovery conditions and data-dependent behaviour in sparse coding and dictionary learning~\citep{mairal2009online,tovsic2011dictionary}.

%%%%%%%%%%%%%%%%%%%%%%
\subsection{Nonnegative Codes}

Our formulation and results consider nonnegative codes. For signed codes, one can replace the nonnegative profile by
\begin{equation}
M^{\rm signed}_{ij}:=\E[z_iz_j\mid a_i=a_j=1],\qquad
A_{ij}:=\E[|z_iz_j|\mid a_i=a_j=1],
\end{equation}
where now $A_{ij}$ is the conditional magnitude. When $A_{ij}>0$, we define the signed code alignment as
\begin{equation}
\kappa_{ij}:=\frac{M^{\rm s}_{ij}}{A_{ij}}\in[-1,1],\qquad
I_{ij}=\rho_{ij}\pi_{ij}\kappa_{ij}A_{ij}.
\end{equation}
Under nonnegative coding, $\kappa_{ij}=1$ and the original profile is recovered. We do not explore this setting in the paper and leave for future work.

%%%%%%%%%%%%%%%%%%%%%%
\subsection{Optimality and consistency}

\fixedsuppoptimality*
\begin{proof}
For fixed support $S$, the least-squares objective is
\begin{equation}
\mathcal{L}(\z_S)
=
\frac{1}{2}\left\|\x-\D_S\z_S\right\|_2^2.
\end{equation}
At an interior optimum, the first-order optimality condition requires
\begin{equation}
\nabla_{\z_S}\mathcal{L}
=
-\D_S^\top\left(\x-\D_S\z_S\right)
=
-\D_S^\top\res
=
0.
\end{equation}
Hence,
\begin{equation}
\D_S^\top\res=0.
\end{equation}
For nonnegative codes, interiority ensures that no inequality constraint is active, i.e., $z_i>0$ for all $i\in S$, so the unconstrained first-order condition applies\footnote{The affine identity assumes unit slope on active coordinates. ReLU- and threshold-based encoders are affine on their active coordinates, while TopK-like encoders are affine within regions where the selected indices do not change~\citep{cunningham2023sparse,bricken2023monosemanticity,gao2025scaling,rajamanoharan2024jumping}. For BatchTopK, this holds conditional on the batch selection mask, hence, not point-wise sparsifier.}.
\end{proof}

\encconsidn*

\begin{proof}
Substitute $\x=\D\z+\res$ into $z_i=\w_i^\top\x+b_i$ on the active region, separate the $j=i$ term, and take the conditional expectation.
\end{proof}

%%%%%%%%%%%%%%%%%%%%%%
\subsection{Encoder-decoder duality}

Consider an untied encoder. The decoder Gram $\D^\top\D$ continues to determine decoder geometry, whereas the cross-Gram $\W^\top\D$ determines inference cross-talk. These two matrices need not have the same off-diagonal structure.

When $\D_S$ has full column rank, the least-squares code and its support-specific linear encoder are
\begin{equation}
    \z_S^*=(\D_S^\top\D_S)^{-1}\D_S^\top\x,
    \qquad
    \W_S^*=\D_S(\D_S^\top\D_S)^{-1},
    \qquad
    (\W_S^*)^\top\D_S=\eye,
    \label{eq:support-dual}
\end{equation}
where $\W_S^*$ is the ideal encoder conditioned on current active support.

%%%%%%%%%%%%%%%%%%%%%%
\subsection{Average co-activation probability pairs under uniformly distributed support}\label{app:barpi}
Here we derive the average co-activation probability pairs when the support is uniformly chosen across data examples. For a $k$-sparse autoencoder. Assume support is uniformly distributed across all features. For activation indicator $a_i$, a $k$-sparse code satisfy
\begin{equation}
\sum_{i\neq j}a_i a_j = \left(\sum_i a_i\right)^2-\sum_i a_i = k(k-1),
\end{equation}
resulting in sum co-activation pair probability of
\begin{equation}
\sum \pi_{ij} = \E[a_i a_j] = k (k-1).
\end{equation}
The average co-activation probability pairs from $\tilde p$ total used features
(i.e., $\tilde p = p - \texttt{\# dead neurons}$) would be
\begin{equation}
\overline{\pi_{ij}} = \frac{1}{\tilde p (\tilde p-1)}\sum \pi_{ij} = \frac{k (k-1)}{\tilde p (\tilde p - 1)}
\end{equation}
where $\tilde p (\tilde p - 1)$ are number of ordered pairs from $\tilde p$ features.

\newpage

%%%%%%%%%%%%%%%%%%%%%%%%%%%%%%%%%%
%%%%%%%%%%%%%%%%%%%%%%%%%%%%%%%%%%
%
\section{Appendix - Proof of \Cref{prop:local-orthogonalization}}
\label{app:feature-dynamics}

\localorth*

\begin{proof}
Since $S=\{i,j\}$ and the concept directions are unit norm, the quantity $\Delta(\D_S)$ reduces to
\begin{equation}
\Delta(\D_S)
=
2(\dvec_i^\top\dvec_j)^2
=
2\rho_{ij}^2.
\end{equation}
It is therefore sufficient to characterize the first-order change in $\rho_{ij}$ induced by one gradient step followed by column normalization.

In the tied, normalized and bias-free setting ($b_i=0$), the active codes satisfy
\begin{equation}
z_i=\dvec_i^\top\x,
\qquad
z_j=\dvec_j^\top\x, \label{eq:codes_tied}
\end{equation}
and the reconstruction is
\begin{equation}
\hat{\x}
=
\dvec_i z_i+\dvec_j z_j \label{eq:rec_k=2}.
\end{equation}
We first derive the normalized update of each active feature, then use it to obtain the corresponding dynamics of $\rho_{ij}$ and, consequently, of $\Delta(\D_S)$.

\paragraph{Normalized update.}

Let $\gvec_k:=\nabla_{\dvec_k}\mathcal{L}$ denote the gradient with respect to an active feature $\dvec_k$. A gradient step of size $\eta$ gives $\dvec_k-\eta\gvec_k$, which is subsequently projected back onto the unit sphere by column normalization. Expanding this normalized update to first order in $\eta$ yields
\begin{equation}
\dvec_k^+
=
\dvec_k
-
\eta
(\eye-\dvec_k\dvec_k^\top)\gvec_k
+
O(\eta^2).
\label{eq:norm_update}
\end{equation}
Thus, normalization removes the radial component of the gradient, and only its component tangent to the unit sphere contributes to the first-order evolution of the feature direction. We refer to Appendix~\ref{appendix:normalization} for the derivation of this expansion.

\paragraph{Projected gradients.}

The next step is to replace the gradients in the normalized update by their expressions for the tied, bias-free model. As derived in Appendix~\ref{app:gradient-computations}, for each active feature $\dvec_k$, we have
\begin{equation}
\gvec_k
=
(\hat{\x}-\x) z_k
+
\alpha_k\x\dvec_k^\top(\hat{\x}-\x)
\end{equation}

Using \Cref{eq:codes_tied}, \Cref{eq:rec_k=2} and the fact that $\alpha_k=1$, we get 
\begin{align}
\gvec_i
&=
(\dvec_i z_i+\dvec_j z_j-\x)z_i
+
\x\,\dvec_i^\top
(\dvec_i z_i+\dvec_j z_j-\x)
\nonumber\\
&=
\dvec_i z_i^2
+
\dvec_j z_i z_j
-
\x z_i
+
\x\left(
z_i+\rho_{ij}z_j-z_i
\right)
\nonumber\\
&=
\dvec_i z_i^2
+
\dvec_j z_i z_j
+
\x\left(\rho_{ij}z_j-z_i\right),
\end{align}
where we used $\|\dvec_i\|_2=1$. Projecting the gradient onto the tangent space of the unit sphere and using
\begin{equation}
(\eye-\dvec_i\dvec_i^\top)\dvec_i=0,
\qquad
(\eye-\dvec_i\dvec_i^\top)\dvec_j
=
\dvec_j-\rho_{ij}\dvec_i,
\end{equation}
we obtain
\begin{align}
(\eye-\dvec_i\dvec_i^\top)\gvec_i
&=
(\dvec_j-\rho_{ij}\dvec_i)z_i z_j
+
\left(\x-\dvec_i(\dvec_i^\top\x)\right)(\rho_{ij}z_j-z_i)\\
&=
(\dvec_j-\rho_{ij}\dvec_i)z_i z_j
+
\left(\x-\dvec_iz_i\right)(\rho_{ij}z_j-z_i)
.
\label{eq:proj-grad-i}
\end{align}

By symmetry,
\begin{equation}
(\eye-\dvec_j\dvec_j^\top)\gvec_j
=
(\dvec_i-\rho_{ij}\dvec_j)z_i z_j
+
(\x-z_j\dvec_j)(\rho_{ij}z_i-z_j)
.
\label{eq:proj-grad-j}
\end{equation}

\paragraph{Correlation dynamics.}

We now use the projected gradients to characterize how the correlation between the two active directions evolves after one normalized gradient step. Recall that $\rho_{ij}=\dvec_i^\top\dvec_j$. Therefore, using the first-order normalized update in \Cref{eq:norm_update}, the correlation after the update is
\begin{equation}
(\dvec_i^+)^\top\dvec_j^+
=
\dvec_i^\top\dvec_j
-
\eta
\dvec_i^\top
(\eye-\dvec_j\dvec_j^\top)\gvec_j
-
\eta
\dvec_j^\top
(\eye-\dvec_i\dvec_i^\top)\gvec_i
+
O(\eta^2).
\end{equation}

Substituting \Cref{eq:proj-grad-i} and \Cref{eq:proj-grad-j} yields the two symmetric contributions
\begin{align}
\dvec_i^\top
(\eye-\dvec_j\dvec_j^\top)\gvec_j
&=
(1-\rho_{ij}^2)z_i z_j + 
(z_i- z_j \rho_{ij})(\rho_{ij}z_i-z_j)
,
\\
\dvec_j^\top
(\eye-\dvec_i\dvec_i^\top)\gvec_i
&=
(1-\rho_{ij}^2)z_i z_j
+
(z_j-z_i\rho_{ij})(\rho_{ij}z_j-z_i).
\end{align}

Noticing that these two contributions are equal, summing them yields
\begin{align}
(\dvec_i^+)^\top\dvec_j^+
&=
\rho_{ij}
-
2\eta \left(
(1-\rho_{ij}^2)z_i z_j
+
(z_i-z_j\rho_{ij})(\rho_{ij}z_i-z_j)
\right)
+
O(\eta^2)
\\
&=
\rho_{ij}
-
2\eta \left(
z_i z_j
-
\rho_{ij}^2 z_i z_j
+
\rho_{ij}z_i^2
-
z_i z_j
-
\rho_{ij}^2 z_i z_j
+
\rho_{ij}z_j^2
\right)
+
O(\eta^2)
\\
&=
\rho_{ij}
-
2\eta \left(
\rho_{ij}z_i^2
+
\rho_{ij}z_j^2
-
2\rho_{ij}^2z_i z_j
\right)
+
O(\eta^2)
\\
&=
\rho_{ij}
-
2\eta \rho_{ij}
\left(
z_i^2
+
z_j^2
-
2\rho_{ij}z_i z_j
\right)
+
O(\eta^2)\label{eq:tiednormQ}
\\
&=
\rho_{ij}
-
2\eta \rho_{ij}
\left(
z_i^2\|\dvec_i\|_2^2
+
z_j^2\|\dvec_j\|_2^2
-
2z_i z_j\dvec_i^\top\dvec_j
\right)
+
O(\eta^2)
\\
&=
\rho_{ij}
-
2\eta \rho_{ij}
\left\|
z_i\dvec_i-z_j\dvec_j
\right\|_2^2
+
O(\eta^2).
\end{align}

Finally, substituting the correlation dynamics into $\Delta(\D_S)=2\rho_{ij}^2$ gives
\begin{align}
\Delta(\D_S^+)-\Delta(\D_S)
&=
2\bigl((\dvec_i^+)^\top\dvec_j^+\bigr)^2
-
2\rho_{ij}^2
\\
&=
2\left(
\rho_{ij}
-
2\eta \rho_{ij}
\left\|
z_i\dvec_i-z_j\dvec_j
\right\|_2^2
+
O(\eta^2)
\right)^2
-
2\rho_{ij}^2
\\
&=
-8\eta \rho_{ij}^2
\left\|
z_i\dvec_i-z_j\dvec_j
\right\|_2^2
+
O(\eta^2).
\end{align}

The first-order term is of order $\eta \rho_{ij}^2$, while the remainder is of
order $\eta^2$. Writing the remainder as $R(\eta)=O(\eta^2)$, there exist
constants $C>0$ and $\eta_0>0$ such that
\begin{equation}
    |R(\eta)|\leq C\eta^2,
    \qquad 0<\eta<\eta_0.
\end{equation}
Hence,
\begin{align}
\Delta(\D_S^+)-\Delta(\D_S)
&\leq
-8\eta \rho_{ij}^2
\left\|
z_i\dvec_i-z_j\dvec_j
\right\|_2^2
+
C\eta^2.
\end{align}
Therefore, the difference is non-positive whenever
\begin{equation}
    C\eta^2
    \leq
    8\eta \rho_{ij}^2
    \left\|
    z_i\dvec_i-z_j\dvec_j
    \right\|_2^2,
\end{equation}
or equivalently,
\begin{equation}
    \eta
    \leq
    \frac{8}{C}
    \rho_{ij}^2
    \left\|
    z_i\dvec_i-z_j\dvec_j
    \right\|_2^2.
\end{equation}
Provided
$\|z_i\dvec_i-z_j\dvec_j\|_2^2>0$, this quantity is a strictly positive
constant for fixed $z_i,z_j,\dvec_i,\dvec_j$. Thus, under the assumption
$\eta\ll \rho_{ij}^2$ and for sufficiently small $\eta<\eta_0$, the above
bound is satisfied. Consequently, the negative first-order term dominates
the $O(\eta^2)$ remainder, and
\begin{equation}
\Delta(\D_S^+)-\Delta(\D_S)\leq 0.
\end{equation}
This proves the result.
\end{proof}

\newpage

\newpage

%%%%%%%%%%%%%%%%%%%%%%%%%%%%%%%%%%
%%%%%%%%%%%%%%%%%%%%%%%%%%%%%%%%%%
%
\section{Appendix - \Cref{prop:local-orthogonalization-unnormalized}}\label{app:unormalized-tied-dynamics}

\begin{restatable}
[Local orthogonalization without normalization]{proposition}{localorthunnorm}
\label{prop:local-orthogonalization-unnormalized}
Consider the tied, bias-free setting with active support $S=\{i,j\}$ and $z_i z_j>0$. 
Let
\[
n_k=\|\dvec_k\|_2,\qquad
\tilde{\dvec}_k=\frac{\dvec_k}{n_k},\qquad
\tilde z_k=\tilde{\dvec}_k^\top\x=\frac{z_k}{n_k},
\qquad
\rho_{ij}=\tilde{\dvec}_i^\top\tilde{\dvec}_j,
\]
and let $\tilde{\D}$ denote the column-normalized version of $\D$. After one gradient step of size $\eta$ applied to the unnormalized dictionary $\D$,
\begin{equation}
\Delta(\tilde{\D}_S^+)-\Delta(\tilde{\D}_S)
=
-8\eta\rho_{ij}
\left[
\rho_{ij}(\tilde z_i^2+\tilde z_j^2)
+
\tilde z_i\tilde z_j
\left(
(n_i^2+n_j^2)(1-\rho_{ij}^2)-2
\right)
\right]
+O(\eta^2),
\label{eq:unnormalized-angular-dynamics}
\end{equation}
where $\D^+$ is the updated, unnormalized dictionary and $\tilde{\D}^+$ its column-normalized version.

Consequently, for sufficiently small $\eta$, $\Delta(\tilde{\D}_S)$ is guaranteed to be non-increasing to first order whenever
\begin{equation}
\begin{cases}
n_i^2+n_j^2 > \dfrac{2}{1+\rho_{ij}},
& 0<\rho_{ij}<1, \\[5pt]
n_i^2+n_j^2 < \dfrac{2}{1+\rho_{ij}},
& -1<\rho_{ij}<0.
\end{cases}
\label{eq:unnormalized-gram-condition}
\end{equation}
At $\rho_{ij}=0$, the first-order change vanishes; at $\rho_{ij}=1$, the two directions remain positively collinear. The case $\rho_{ij}=-1$ is incompatible with $z_i z_j>0$ for nonzero tied codes.
\end{restatable}

Importantly, the gradient update itself is performed on the unnormalized dictionary $\D$. We evaluate the resulting change through its column-normalized version $\tilde{\D}$ because we are interested in the evolution of the angle between feature directions, independently of changes in their column norms.

\begin{proof}
Since we are interested in the angular dynamics, we track the normalized
feature directions. After one gradient step on the unnormalized dictionary,
\begin{equation}
\tilde{\dvec}_i^+
=
\frac{\dvec_i^+}{\|\dvec_i^+\|_2}
=
\frac{\dvec_i-\eta\nabla_{\dvec_i}\mathcal{L}}
{\|\dvec_i-\eta\nabla_{\dvec_i}\mathcal{L}\|_2}.
\end{equation}

Dividing the numerator and denominator by $n_i$ and using
\Cref{eq:taylor} gives
\begin{align}
\tilde{\dvec}_i^+
&=
\frac{
\tilde{\dvec}_i-\frac{\eta}{n_i}\nabla_{\dvec_i}\mathcal{L}
}{
\|
\tilde{\dvec}_i-\frac{\eta}{n_i}\nabla_{\dvec_i}\mathcal{L}
\|_2
}
\\
&=
\left(
\tilde{\dvec}_i-\frac{\eta}{n_i}\nabla_{\dvec_i}\mathcal{L}
\right)
\left(
1
+
\frac{\eta}{n_i}
\tilde{\dvec}_i^\top
\nabla_{\dvec_i}\mathcal{L}
+
O(\eta^2)
\right)
\\
&=
\tilde{\dvec}_i
-
\frac{\eta}{n_i}
\left(
\eye-\tilde{\dvec}_i\tilde{\dvec}_i^\top
\right)
\nabla_{\dvec_i}\mathcal{L}
+
O(\eta^2).
\end{align}
Thus, to first order, only the component of the gradient orthogonal to
$\tilde{\dvec}_i$ changes the feature direction, while the radial component
changes only its norm.

Substituting the gradient from \Cref{eq:gradient_di_tied} yields
\begin{align}
\nabla_{\dvec_i}\mathcal{L}
&=
(\hat{\x}-\x)z_i
+
\x\dvec_i^\top(\hat{\x}-\x)
\\
&=
\left(
z_i\dvec_i+z_j\dvec_j-\x
\right) z_i
+
\x\dvec_i^\top\left(
z_i\dvec_i+z_j\dvec_j-\x
\right)
\\
&=
z_i^2\dvec_i
+
z_i z_j\dvec_j
-
z_i\x
+
\x\left(
z_i n_i^2
+
z_j n_i n_j\rho_{ij}
-
z_i
\right)
\\
&=
z_i^2\dvec_i
+
z_i z_j\dvec_j
+
\left[
(n_i^2-2)z_i
+
n_i n_j\rho_{ij}z_j
\right]\x .
\end{align}

We now project each term separately. First, since
\(\dvec_i=n_i\tilde{\dvec}_i\),
\begin{equation}
    \left(
        \eye-\tilde{\dvec}_i\tilde{\dvec}_i^\top
    \right)
    z_i^2\dvec_i
    =
    z_i^2
    \left(
        \eye-\tilde{\dvec}_i\tilde{\dvec}_i^\top
    \right)
    n_i\tilde{\dvec}_i=
    0 .
\end{equation}
Second,
\begin{align}
    \left(
        \eye-\tilde{\dvec}_i\tilde{\dvec}_i^\top
    \right)
    z_i z_j\dvec_j
    &=
    z_i z_j
    \left(
        \eye-\tilde{\dvec}_i\tilde{\dvec}_i^\top
    \right)
    n_j\tilde{\dvec}_j
    \\
    &=
    n_j z_i z_j
    \left(
        \tilde{\dvec}_j
        -
        \tilde{\dvec}_i\tilde{\dvec}_i^\top\tilde{\dvec}_j
    \right)
    =
    n_j z_i z_j
    \left(
        \tilde{\dvec}_j
        -
        \rho_{ij}\tilde{\dvec}_i
    \right).
\end{align}
Finally,
\begin{align}
\left(
\eye-\tilde{\dvec}_i\tilde{\dvec}_i^\top
\right)
\x
\left[
(n_i^2-2)z_i
+
n_i n_j\rho_{ij}z_j
\right]
&=
\left[
\x
-
\tilde{\dvec}_i
\left(
\tilde{\dvec}_i^\top\x
\right)
\right]
\left[
(n_i^2-2)z_i
+
n_i n_j\rho_{ij}z_j
\right]
\\
&=
\left(
\x-\tilde z_i\tilde{\dvec}_i
\right)
\left[
(n_i^2-2)z_i
+
n_i n_j\rho_{ij}z_j
\right].
\end{align}

Therefore,
\begin{equation}
\left(
\eye-\tilde{\dvec}_i\tilde{\dvec}_i^\top
\right)
\nabla_{\dvec_i}\mathcal{L}
=
n_j z_i z_j
\left(
\tilde{\dvec}_j
-
\rho_{ij}\tilde{\dvec}_i
\right)
+
\left[
(n_i^2-2)z_i
+
n_i n_j\rho_{ij}z_j
\right]
\left(
\x-\tilde z_i\tilde{\dvec}_i
\right).
\end{equation}

Substituting this expression into the first-order update of the normalized direction gives
\begin{equation}
\tilde{\dvec}_i^+
=
\tilde{\dvec}_i
-
\frac{\eta}{n_i}
\Bigg[
n_j z_i z_j
\left(
\tilde{\dvec}_j-\rho_{ij}\tilde{\dvec}_i
\right)
+
\left[
(n_i^2-2)z_i+n_i n_j\rho_{ij}z_j
\right]
\left(
\x-\tilde z_i\tilde{\dvec}_i
\right)
\Bigg]
+
O(\eta^2).
\end{equation}

By symmetry, the corresponding update for atom $j$ is
\begin{equation}
\tilde{\dvec}_j^+
=
\tilde{\dvec}_j
-
\frac{\eta}{n_j}
\Bigg[
n_i z_i z_j
\left(
\tilde{\dvec}_i-\rho_{ij}\tilde{\dvec}_j
\right)
+
\left[
(n_j^2-2)z_j+n_i n_j\rho_{ij}z_i
\right]
\left(
\x-\tilde z_j\tilde{\dvec}_j
\right)
\Bigg]
+
O(\eta^2).
\end{equation}

Given the required inner products, 
\begin{alignat}{2}
\tilde{\dvec}_j^\top\left(
\tilde{\dvec}_j
-
\rho_{ij}\tilde{\dvec}_i
\right)
&=
1-\rho_{ij}^2,
\qquad
&
\tilde{\dvec}_j^\top\left(
\x-\tilde z_i\tilde{\dvec}_i
\right)
&=
\tilde z_j-\rho_{ij}\tilde z_i,
\\
\tilde{\dvec}_i^\top
\left(
\tilde{\dvec}_i
-
\rho_{ij}\tilde{\dvec}_j
\right)
&=
1-\rho_{ij}^2,
\qquad
&
\tilde{\dvec}_i^\top
\left(
\x-\tilde z_j\tilde{\dvec}_j
\right)
&=
\tilde z_i-\rho_{ij}\tilde z_j.
\end{alignat}

We can now compute the first-order evolution of the normalized correlation
\(\rho_{ij}=\tilde{\dvec}_i^\top\tilde{\dvec}_j\). Using
\begin{equation}
    \rho_{ij}^+
    =
    (\tilde{\dvec}_i^+)^\top \tilde{\dvec}_j^+ ,
\end{equation}
and keeping only first-order terms in \(\eta\), we get
\begin{align}
\rho_{ij}^+
&=
\rho_{ij}
-
\frac{\eta}{n_i}
\Bigg[
n_j z_i z_j
\tilde{\dvec}_j^\top\left(
\tilde{\dvec}_j
-
\rho_{ij}\tilde{\dvec}_i
\right)
+
\left[
(n_i^2-2)z_i
+
n_i n_j\rho_{ij}z_j
\right]
\tilde{\dvec}_j^\top\left(
\x-\tilde z_i\tilde{\dvec}_i
\right)
\Bigg]
\nonumber\\
&\quad
-
\frac{\eta}{n_j}
\Bigg[
n_i z_i z_j
\tilde{\dvec}_i^\top
\left(
\tilde{\dvec}_i
-
\rho_{ij}\tilde{\dvec}_j
\right)
+
\left[
(n_j^2-2)z_j
+
n_i n_j\rho_{ij}z_i
\right]
\tilde{\dvec}_i^\top
\left(
\x-\tilde z_j\tilde{\dvec}_j
\right)
\Bigg]
+
O(\eta^2)\nonumber
\\
&=
\rho_{ij}
-
\eta z_i z_j
\left(
\frac{n_j}{n_i}
+
\frac{n_i}{n_j}
\right)
\left(
1-\rho_{ij}^2
\right)
\nonumber\\
&\quad
-
\frac{\eta}{n_i}
\left[
(n_i^2-2)z_i
+
n_i n_j\rho_{ij}z_j
\right]
\left(
\tilde z_j-\rho_{ij}\tilde z_i
\right)
\nonumber\\
&\quad
-
\frac{\eta}{n_j}
\left[
(n_j^2-2)z_j
+
n_i n_j\rho_{ij}z_i
\right]
\left(
\tilde z_i-\rho_{ij}\tilde z_j
\right)
+
O(\eta^2).
\nonumber
\\
&=
\rho_{ij}
-
\eta \tilde z_i\tilde z_j
\left(
n_i^2+n_j^2
\right)
\left(
1-\rho_{ij}^2
\right)
\nonumber\\
&\quad
-
\eta
\left[
(n_i^2-2)\tilde z_i
+
n_j^2\rho_{ij}\tilde z_j
\right]
\left(
\tilde z_j-\rho_{ij}\tilde z_i
\right)
\nonumber\\
&\quad
-
\eta
\left[
(n_j^2-2)\tilde z_j
+
n_i^2\rho_{ij}\tilde z_i
\right]
\left(
\tilde z_i-\rho_{ij}\tilde z_j
\right)
+
O(\eta^2)
\\
&=
\rho_{ij}
-
\eta \tilde z_i\tilde z_j
(n_i^2+n_j^2)
(1-\rho_{ij}^2)
\nonumber\\
&\quad
-
\eta
\Big[
(n_i^2-2)\tilde z_i\tilde z_j
-
(n_i^2-2)\rho_{ij}\tilde z_i^2
+
n_j^2\rho_{ij}\tilde z_j^2
-
n_j^2\rho_{ij}^2\tilde z_i\tilde z_j
\Big]
\nonumber\\
&\quad
-
\eta
\Big[
(n_j^2-2)\tilde z_i\tilde z_j
-
(n_j^2-2)\rho_{ij}\tilde z_j^2
+
n_i^2\rho_{ij}\tilde z_i^2
-
n_i^2\rho_{ij}^2\tilde z_i\tilde z_j
\Big]
+
O(\eta^2)
\\
&=
\rho_{ij}
-
\eta \tilde z_i\tilde z_j
(n_i^2+n_j^2)
(1-\rho_{ij}^2)
\nonumber\\
&\quad
-
\eta
\Big[
(n_i^2+n_j^2-4)\tilde z_i\tilde z_j
-
(n_i^2+n_j^2)\rho_{ij}^2\tilde z_i\tilde z_j
+
2\rho_{ij}
(\tilde z_i^2+\tilde z_j^2)
\Big]
+
O(\eta^2)
\\
&=
\rho_{ij}
-
\eta
\Big[
(n_i^2+n_j^2)(1-\rho_{ij}^2)\tilde z_i\tilde z_j
+
(n_i^2+n_j^2-4)\tilde z_i\tilde z_j
\nonumber\\
&\qquad\qquad
-
(n_i^2+n_j^2)\rho_{ij}^2\tilde z_i\tilde z_j
+
2\rho_{ij}(\tilde z_i^2+\tilde z_j^2)
\Big]
+
O(\eta^2)
\\
&=
\rho_{ij}
-
2\eta
\Big[
\rho_{ij}
(\tilde z_i^2+\tilde z_j^2)
+
\tilde z_i\tilde z_j
\big(
(n_i^2+n_j^2)(1-\rho_{ij}^2)-2
\big)
\Big]
+
O(\eta^2). 
\end{align}
Setting $n_i=n_j=1$ recovers the normalized case in
\Cref{eq:tiednormQ}. 

Since
\begin{equation}
\Delta(\tilde{\D}_S)=2\rho_{ij}^2,
\end{equation}
the first-order change in interference measure is
\begin{align}
\Delta(\tilde{\D}_S^+)-\Delta(\tilde{\D}_S)
&=
-8\eta\rho_{ij}
\Big[
\rho_{ij}
(\tilde z_i^2+\tilde z_j^2)
\nonumber\\
&\qquad\qquad
+
\tilde z_i\tilde z_j
\big(
(n_i^2+n_j^2)(1-\rho_{ij}^2)-2
\big)
\Big]
+
O(\eta^2).
\end{align}

Hence, the interference measure decreases to first order whenever
\begin{equation}
\rho_{ij}
\Big[
\rho_{ij}(\tilde z_i^2+\tilde z_j^2)
+
\tilde z_i\tilde z_j
\big(
(n_i^2+n_j^2)(1-\rho_{ij}^2)-2
\big)
\Big]
>0.
\end{equation}
Using
\begin{align}
&\rho_{ij}
\left(
\tilde z_i^2+\tilde z_j^2
\right)
+
\tilde z_i\tilde z_j
\left[
(n_i^2+n_j^2)(1-\rho_{ij}^2)-2
\right]
\\
&=
\rho_{ij}
\left(
\tilde z_i^2+\tilde z_j^2-2\tilde z_i\tilde z_j
\right)
+
\tilde z_i\tilde z_j
\left[
(n_i^2+n_j^2)(1-\rho_{ij}^2)
-2
+2\rho_{ij}
\right]
\\
&=
\rho_{ij}
\left(
\tilde z_i-\tilde z_j
\right)^2
+
\tilde z_i\tilde z_j
\left[
(n_i^2+n_j^2)(1-\rho_{ij}^2)
-2(1-\rho_{ij})
\right]
\\
&=
\rho_{ij}
\left(
\tilde z_i-\tilde z_j
\right)^2
+
(1-\rho_{ij})
\left[
(n_i^2+n_j^2)(1+\rho_{ij})-2
\right]
\tilde z_i\tilde z_j.
\end{align}
and assuming $\tilde z_i\tilde z_j>0$, the exact threshold is
\begin{equation}
\begin{cases}
n_i^2+n_j^2>\;\dfrac{2}{1+\rho_{ij}}
-
\dfrac{\rho_{ij}}{1-\rho_{ij}^2}
\dfrac{(\tilde z_i-\tilde z_j)^2}{\tilde z_i\tilde z_j},
& 0<\rho_{ij}<1,\\[10pt]
n_i^2+n_j^2<\;\dfrac{2}{1+\rho_{ij}}
-
\dfrac{\rho_{ij}}{1-\rho_{ij}^2}
\dfrac{(\tilde z_i-\tilde z_j)^2}{\tilde z_i\tilde z_j},
& -1<\rho_{ij}<0.
\end{cases}
\end{equation}
The correction term always shifts the threshold in the direction that makes the inequality easier to satisfy. Therefore, a $\tilde z_i,\tilde z_j$-independent sufficient condition is
\begin{equation}
\begin{cases}
n_i^2+n_j^2 > \dfrac{2}{1+\rho_{ij}},
& 0<\rho_{ij}<1,\\[5pt]
n_i^2+n_j^2 < \dfrac{2}{1+\rho_{ij}},
& -1<\rho_{ij}<0.
\end{cases}
\end{equation}
This guarantees a first-order decrease for any $\tilde z_i\tilde z_j>0$.

The boundary cases are immediate. At $\rho_{ij}=0$, the first-order change vanishes. At $\rho_{ij}=1$, the normalized feature directions remain positively collinear, so $\rho_{ij}^+=1$ exactly. The case $\rho_{ij}=-1$ is incompatible with $\tilde z_i\tilde z_j>0$ for nonzero tied codes.
\end{proof}

When $n_i=n_j=1$, the sufficient condition holds for all $-1<\rho_{ij}<1$, $\rho_{ij}\neq0$, recovering the normalized result. Without normalization, however, the angular dynamics also depend on $n_i^2+n_j^2$. Thus, gradient descent need not enforce local orthogonalization: feature norms provide an additional degree of freedom through which interference can be reduced.

\newpage
\section{Appendix - Column Normalization}
\label{appendix:normalization}
The mean squared error objective does not uniquely determine the scale of the decoder directions, since scaling a decoder feature $\dvec_i$ by a constant and inversely scaling its activation $z_i$ leaves the reconstruction unchanged. To remove this scaling ambiguity, sparse dictionary learning commonly constrain each decoder feature to have unit norm. In practice, SAE implementations typically enforce this constraint through two operations applied sequentially at each training step \citep{cunningham2023sparse,bricken2023monosemanticity,gao2025scaling}. First, the component of the gradient parallel to each decoder column is projected out, so that the optimizer update is tangent to the unit sphere. After the optimizer step, each decoder column is explicitly renormalized to unit norm.  

The gradient projection alone does not enforce the constraint exactly: even a tangent update can cause the feature norms to drift over long training horizons \citep{douglas2002gradient}. It nevertheless ensures that the gradient passed to adaptive optimizers such as Adam \citep{kingma2014adam} is consistent with the unit-norm constraint, rather than relying on a subsequent renormalization to correct an unconstrained update. This has been observed empirically to yield small but consistent improvements in the optimization objective \citep{bricken2023monosemanticity}. The subsequent explicit renormalization then corrects the remaining norm drift and restores unit norm exactly after every optimizer step.

In the following, we make the relationship between these two sequential operations explicit by deriving the first-order expansion of a renormalized gradient update. We show that, to the first order, explicit renormalization itself induces a projection of the gradient onto the tangent space of the unit sphere, with the discrepancy appearing only at higher order.

Consider a dictionary feature $\dvec_i$ satisfying $\|\dvec_i\|_2=1$, and denote
its gradient by
\begin{equation}
    \gvec_i := \nabla_{\dvec_i}\mathcal{L}.
\end{equation}
After a gradient step with step size $\eta>0$ followed by renormalization, the
updated feature is
\begin{equation}
    \dvec_i^+
    =
    \frac{\dvec_i-\eta \gvec_i}
    {\|\dvec_i-\eta \gvec_i\|_2}.
\end{equation}
The normalization factor satisfies
\begin{align}
    \|\dvec_i-\eta \gvec_i\|_2
    &=
    \sqrt{
    (\dvec_i-\eta \gvec_i)^\top
    (\dvec_i-\eta \gvec_i)
    } \\
    &=
    \sqrt{
    1
    -2\eta \gvec_i^\top\dvec_i
    +\eta^2 \gvec_i^\top\gvec_i
    }.
\end{align}
Define
\begin{equation}
   \alpha_i
    :=
    2\eta \gvec_i^\top\dvec_i
    -
    \eta^2 \gvec_i^\top\gvec_i .
\end{equation}
Then
\begin{equation}
    \frac{1}{\|\dvec_i-\eta \gvec_i\|_2}
    =
    (1-\alpha_i)^{-1/2}. \label{eq:taylor}
\end{equation}
Using Taylor's theorem with Lagrange remainder,
\begin{equation}
    (1-   \zeta_i)^{-1/2}
    =
    1
    +
    \frac{   \zeta_i}{2}
    +
    \frac{3   \zeta_i^2}{8}(1-\xi)^{-5/2},
\end{equation}
for some $\xi$ between $0$ and $   \zeta_i$. Substituting the expression for $   \zeta_i$, we get
\begin{align}
    \dvec_i^+
    &=
    (\dvec_i-\eta \gvec_i)
    \Bigg[
    1
    +
    \eta \gvec_i^\top\dvec_i
    -
    \frac{\eta^2}{2}\gvec_i^\top\gvec_i
    \nonumber \\
    &\qquad\qquad
    +
    \frac{3}{8}
    \left(
        2\eta \gvec_i^\top\dvec_i
        -
        \eta^2 \gvec_i^\top\gvec_i
    \right)^2
    (1-\xi)^{-5/2}
    \Bigg].
\end{align}
Keeping the first-order terms in $\eta$ yields
\begin{align}
    \dvec_i^+
    &=
    \dvec_i
    -
    \eta \gvec_i
    +
    \eta \dvec_i\dvec_i^\top\gvec_i
    +
    R_i \\
    &=
    \dvec_i
    -
    \eta
    \left(
        \eye-\dvec_i\dvec_i^\top
    \right)\gvec_i
    +
    R_i,
\end{align}
where $R_i$ collects the second- and higher-order terms. More explicitly,
\begin{align}
    R_i
    &=
    -\eta^2 \gvec_i(\gvec_i^\top\dvec_i)
    -
    \frac{\eta^2}{2}\dvec_i(\gvec_i^\top\gvec_i)
    \nonumber \\
    &\quad
    +
    \frac{3}{8}
    (\dvec_i-\eta\gvec_i)
    \left(
        2\eta \gvec_i^\top\dvec_i
        -
        \eta^2 \gvec_i^\top\gvec_i
    \right)^2
    (1-\xi)^{-5/2},
\end{align}
for some $\xi$ between $0$ and
$2\eta \gvec_i^\top\dvec_i-\eta^2\gvec_i^\top\gvec_i$.

Thus, to first order in the learning rate, explicit renormalization is
equivalent to replacing the gradient $\gvec_i$ by its projection onto the
tangent space of the unit sphere at $\dvec_i$:
\begin{equation}
    \gvec_i
    \longmapsto
    \left(
        \eye-\dvec_i\dvec_i^\top
    \right)\gvec_i.
\end{equation}
Therefore, explicit normalization and tangent-space gradient projection induce
the same leading-order dynamics on the unit sphere. Their difference is
captured by the higher-order remainder $R_i=O(\eta^2)$.

\paragraph{$\ell_1$ regularization.}
The unit-norm constraint is particularly important for SAEs trained with $\ell_1$ regularization, as otherwise the sparsity penalty can be reduced by increasing the decoder norms while proportionally decreasing the codes, without changing the reconstruction. Recent work has proposed removing the unit-norm constraint by instead using the reparameterization-invariant regularizer $\lambda \sum_{i\in S} \|\dvec_i\|_2 z_i$ \citep{rajamanoharan2024jumping,Conerly2024}.

%%%%%%%%%%%%%%%%%%%%%%%%%%%%%%%%%%
%%%%%%%%%%%%%%%%%%%%%%%%%%%%%%%%%%
%

\newpage 

\section{Appendix - Gradient computations}
\label{app:gradient-computations}

In this appendix, we derive the gradient expressions used throughout the proofs. We consider a single input $\x$ and define its active support as
\begin{equation}
S(\x) := \{i : z_i(\x) \neq 0\}.
\end{equation}
We treat the active support as locally fixed throughout the derivation, such that inactive features receive zero gradient. The reconstruction is then
\begin{equation}
\hat{\x}
=
\sum_{i \in S(\x)} \dvec_i z_i,
\end{equation}
and we consider the reconstruction loss
\begin{equation}
\mathcal{L}(\x)
=
\frac{1}{2}\|\hat{\x}-\x\|_2^2.
\end{equation}

In the tied setting, the encoder and decoder weights are shared, such that $\W = \D$. For an active feature $i\in S(\x)$, the activation is
\begin{equation}
z_i = \dvec_i^\top \x + b_i,
\end{equation}
and the reconstruction is therefore
\begin{equation}
\hat{\x}
=
\sum_{i\in S(\x)} \dvec_i z_i
=
\sum_{i\in S(\x)}
\dvec_i\left(\dvec_i^\top\x+b_i\right).
\end{equation}
The reconstruction loss is
\begin{equation}
\mathcal{L}(\x)
=
\frac{1}{2}\|\hat{\x}-\x\|_2^2
=
\frac{1}{2}
\big\|
\sum_{i\in S(\x)}
\dvec_i\left(\dvec_i^\top\x+b_i\right)-\x
\big\|_2^2.
\end{equation}

\paragraph{Gradient with respect to a dictionary feature.}
Because the weights are tied, $\dvec_i$ affects the reconstruction in two ways: directly through the decoder and indirectly through the activation $z_i$. For an active feature $i \in S(\x)$, we have
\begin{equation}
\frac{\partial z_i}{\partial \dvec_i}
=
\x^\top.
\end{equation}
Using the product rule, the derivative of the reconstruction with respect to $\dvec_i$ is therefore
\begin{equation}
\frac{\partial \hat{\x}}{\partial \dvec_i}
=
\frac{\partial}{\partial \dvec_i}
\left(\dvec_i z_i\right) \\
=
\left(\frac{\partial \dvec_i}{\partial \dvec_i}\right) z_i 
+
\dvec_i\left(\frac{\partial z_i}{\partial \dvec_i}\right)
=
\eye z_i
+
\dvec_i \x^\top.
\end{equation}

Since the gradient of the reconstruction loss with respect to $\hat{\x}$ is $\hat{\x}-\x$, applying the chain rule yields
\begin{equation}
\frac{\partial \mathcal{L}}{\partial \dvec_i}
=
(\hat{\x}-\x)^\top
\left(
\frac{\partial\hat{\x}}{\partial\dvec_i}
\right)
=
(\hat{\x}-\x)^\top\left(
\eye z_i+\dvec_i\x^\top
\right)
=
(\hat{\x}-\x)^\top z_i
+
\left((\hat{\x}-\x)^\top\dvec_i \right) \x^\top.
\end{equation}
Hence,
\begin{equation}
    \nabla_{\dvec_i}\mathcal{L}
    =
    \begin{cases}
    (\hat{\x}-\x) z_i
    +
    \x\dvec_i^\top(\hat{\x}-\x)
    & i\in S(\x),
    \\
    0,
    & i\notin S(\x).
    \end{cases} \label{eq:gradient_di_tied}
\end{equation}
The first term corresponds to the contribution of $\dvec_i$ through the decoder, while the second arises from its contribution through the tied encoder.

\newpage
\section{Appendix - Experimental Details}
\label{app:exp-details}

\textbf{Model and data.}\quad
We use Pythia-160M-deduped \citep{biderman2023pythiasuiteanalyzinglarge}, taking residual-stream activations after layer 8 (\texttt{blocks.8.hook\_resid\_post}, $m=768$). Activations come from the uncopyrighted Pile \citep{gao2020pile800gbdatasetdiverse}, using 1024-token contexts with a prepended BOS token, and are cached with SAELens \citep{bloom2024saetrainingcodebase}. We center activations by the training-set geometric median as in \citep{gao2025scaling}, but compute $\bvec_{\text{pre}}$ once and keep it fixed. Throughout, $\x=\x_{\mathrm{raw}}-\bvec_{\text{pre}}$; no other normalization is applied.

\textbf{Training.}\quad
Each SAE is trained for 200M tokens (48{,}828 steps of $2^{14}$ tokens) on disjoint activations and evaluated on a fixed, disjoint 2M-token holdout set.\footnote{\citep{karvonen2025saebench} report that many SAE evaluation metrics achieve most of their performance by 50M training tokens.} We use Adam \citep{kingma2014adam} with learning rate $2\times10^{-4}$, $(\beta_1,\beta_2)=(0.9,0.999)$, $\epsilon=10^{-8}$, no weight decay, schedule, or warmup.

\textbf{Architecture.}\quad
All models use
\begin{equation}
\z=\sigma\!\left(\mathrm{ReLU}\!\left(\W^\top\x+\bvec\right)\right),
\qquad
\hat{\x}=\D\z,
\end{equation}
with $\D\in\R^{m\times p}$ unit-norm columnwise, $p=16{,}384$, and $k=40$. We compare TopK \citep{gao2025scaling}, BatchTopK \citep{bussmann2024batchtopk}, JumpReLU \citep{rajamanoharan2024jumping}, and Matryoshka \citep{bussmann2025learning}. For JumpReLU, we enforce sparsity through L0 target regularization. For Matryoshka, we use the SAEBench block boundaries at $512$, $1536$, $3584$, and $7680$ features \citep{karvonen2025saebench}.

\textbf{Initialization and constraints.}\quad
$\D$ is Gaussian-initialized and column-normalized, with $\W=\D$ and $\bvec=\mathbf{0}$ at initialization. Decoder columns remain unit norm by projecting out the radial gradient component and renormalizing after each update (\Cref{appendix:normalization}).

We parameterize the encoder weight as $\W=\boldsymbol{\alpha}\odot\D$ or $\W=\D$, where $\boldsymbol{\alpha}\in\R^p$ is a separately learnable, per-feature gain (initialized as $\boldsymbol{\alpha}=\mathbf{1}$) that rescales each column of $\D$ independently of its unit-norm constraint. We consider five settings:
\begin{enumerate}[label=(\roman*)]
    \item \emph{tied}: $\W=\D$, $\boldsymbol{\alpha}=\mathbf{1}$, $\bvec=\mathbf{0}$;
    \item \emph{tied + bias}: $\W=\D$, $\boldsymbol{\alpha}=\mathbf{1}$, with $\bvec$ learned;
    \item \emph{tied + gain}: $\W=\boldsymbol{\alpha}\odot\D$, with $\boldsymbol{\alpha}$ learned and $\bvec=\mathbf{0}$;
    \item \emph{tied + gain + bias}: $\W=\boldsymbol{\alpha}\odot\D$, with both $\boldsymbol{\alpha}$ and $\bvec$ learned;
    \item \emph{untied}: $\W$ is learned independently, with $\boldsymbol{\alpha}=\mathbf{1}$ and $\bvec=\mathbf{0}$.
\end{enumerate}
In the \emph{tied + gain} settings, this lets us isolate the effect of decoupling the encoder's per-feature scale from the decoder's direction, without also decoupling the direction itself. In the untied setting, $\W$ already carries per-feature scale, so we do not additionally learn $\boldsymbol{\alpha}$.

\textbf{Seeds.}\quad
\Cref{fig:local-orthogonalization} averages over three seeds, varying both initialization and token order; all other experiments use a single seed. Although the aggregate geometric and effective interference patterns in this figure are stable across seeds, the learned dictionaries differ substantially, consistent with previous observations \citep{fel2025archetypal, nelson2026toward, song-etal-2026-mechanistic, paulo2026sparse}. For the runs in \Cref{fig:local-orthogonalization}, Table~\ref{tab:cross-seed} quantifies this using mean maximum cosine similarity,
\begin{equation}
\mathrm{MMCS}=\frac{1}{p}\sum_i \max_j \cos(\dvec_i^A,\dvec_j^B),
\end{equation}
and its absolute-value counterpart.

\begin{table}[H]
\centering
\small
\begin{tabular}{lrrrrr}
\toprule
& $\mathrm{MMCS}$ & median & $\mathrm{MMCS}_{|\cdot|}$ & $>\!0.9$ & $>\!0.7$ \\
\midrule
JumpReLU   & $0.711$ & $0.785$ & $0.712$ & $24.8\%$ & $61.1\%$ \\
TopK       & $0.548$ & $0.533$ & $0.552$ & $8.9\%$  & $31.2\%$ \\
BatchTopK  & $0.436$ & $0.335$ & $0.444$ & $6.0\%$  & $18.5\%$ \\
Matryoshka & $0.425$ & $0.332$ & $0.433$ & $3.1\%$  & $15.8\%$ \\
\bottomrule
\end{tabular}
\caption{Cross-seed dictionary agreement for the runs in \Cref{fig:local-orthogonalization}, averaged over seed pairs $(0,1)$, $(0,2)$, and $(1,2)$.}
\label{tab:cross-seed}
\end{table}

\textbf{Evaluation.}\quad
We report fraction of variance unexplained (FVU), average sparsity $\mathbb E\|\z\|_0$, and the fraction of features that never activate on the holdout set. Sparsity is closely matched across settings, while FVU generally decreases as constraints are relaxed.

\begin{table}[H]
\centering
\footnotesize
\setlength{\tabcolsep}{3.5pt}
\begin{tabular}{lrrrrrrrrrrrr}
\toprule
& \multicolumn{3}{c}{TopK} & \multicolumn{3}{c}{BatchTopK}
& \multicolumn{3}{c}{Matryoshka} & \multicolumn{3}{c}{JumpReLU} \\
\cmidrule(lr){2-4}\cmidrule(lr){5-7}\cmidrule(lr){8-10}\cmidrule(lr){11-13}
Constraint & FVU & L0 & Dead & FVU & L0 & Dead & FVU & L0 & Dead & FVU & L0 & Dead \\
\midrule
(i) Tied               & .118 & 40.0 & 0.00 & .124 & 40.9 & 0.00 & .126 & 40.7 & 0.00 & .100 & 45.6 & 0.02 \\
(ii) Tied + Bias       & .086 & 40.0 & 0.00 & .085 & 40.1 & 0.00 & .091 & 40.1 & 0.00 & .086 & 45.5 & 0.48 \\
(iii) Tied + Gain      & .091 & 40.0 & 0.00 & .092 & 40.0 & 0.12 & .096 & 40.2 & 0.01 & .087 & 45.4 & 0.62 \\
(iv) Tied + Gain + Bias& .085 & 40.0 & 0.01 & .084 & 40.0 & 0.02 & .090 & 40.1 & 0.01 & .085 & 45.4 & 0.90 \\
(v) Untied             & .083 & 40.0 & 0.54 & .084 & 40.8 & 5.41 & .088 & 40.6 & 6.05 & .083 & 45.0 & 17.62 \\
\bottomrule
\end{tabular}
\caption{Holdout evaluation after 200M training tokens. ``Dead'' is the percentage of features that never activate on the holdout set.}
\label{tab:eval}
\end{table}

\newpage

\section{Appendix - Additional Results}\label{app:additional results}
\subsection{Extended results}
\begin{figure}[H]
    \centering
    \includegraphics[height=0.77\linewidth, angle=-90]{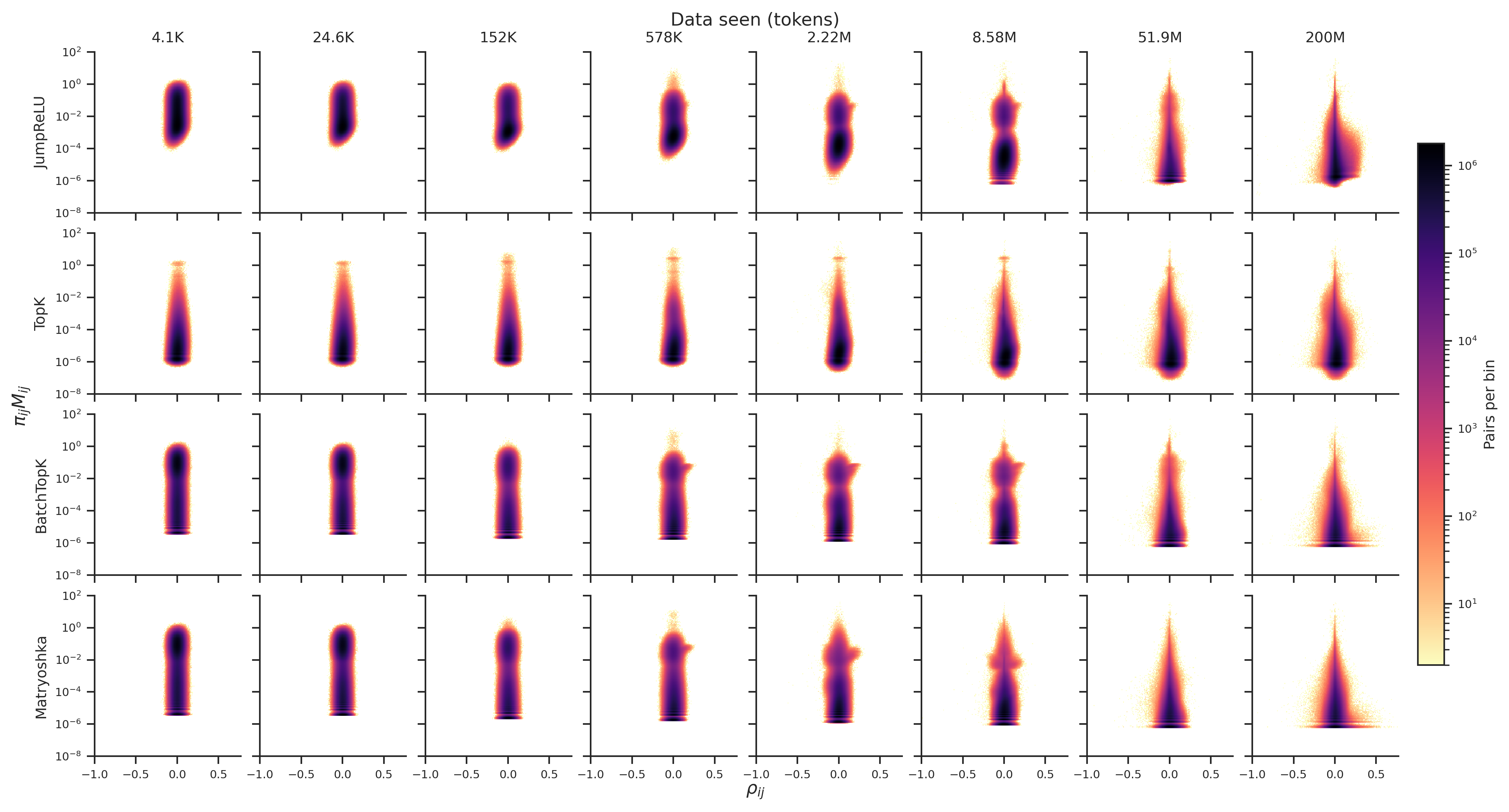}
    \caption{\textbf{Transition of interference factors as training progress.} This figure complements~\Cref{fig:local-orthogonalization}.}
    \label{fig:local-orthogonalization-flame}
\end{figure}

\begin{figure}[H]
    \centering
    \includegraphics[width=\linewidth]{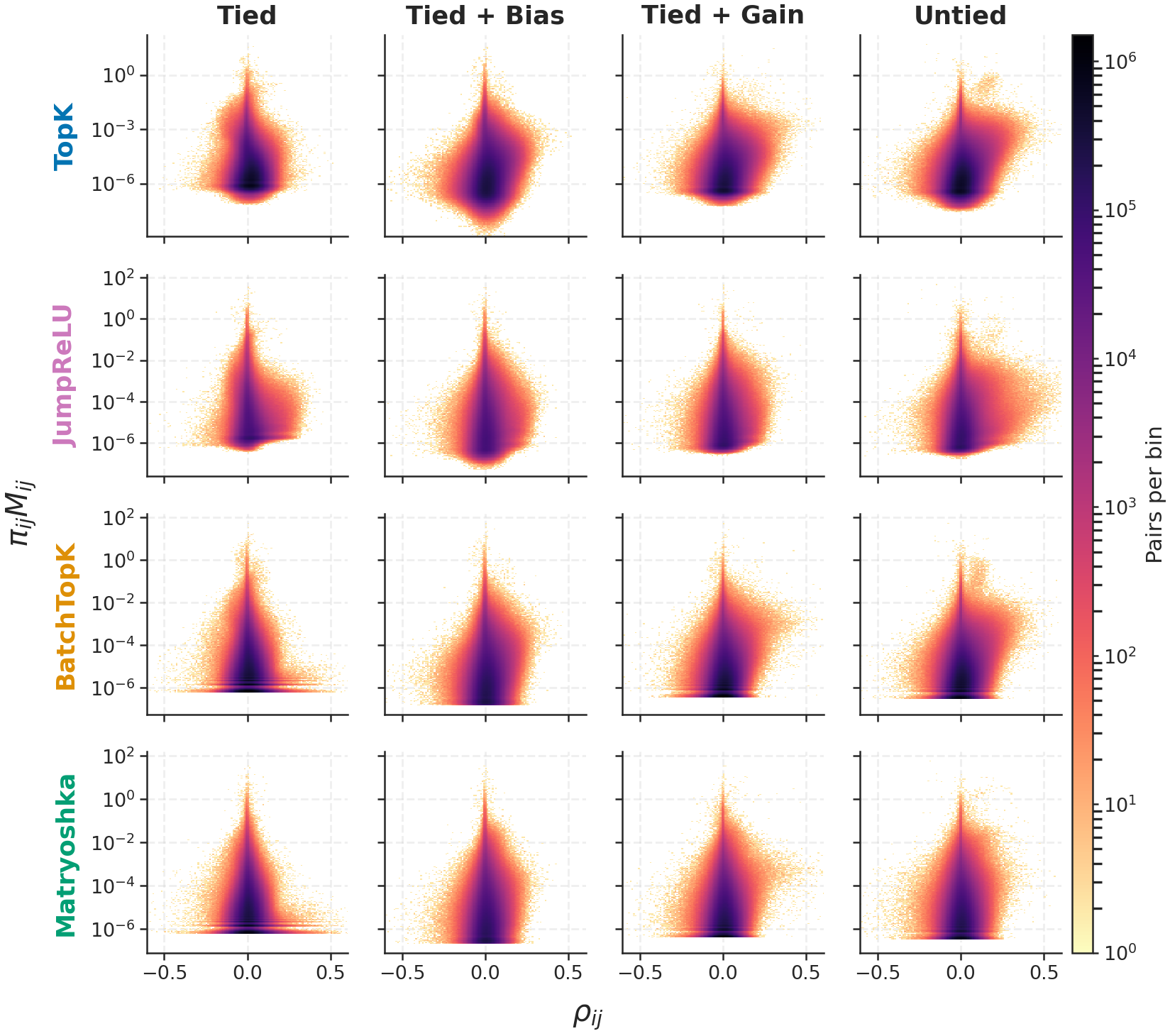}
    \caption{\textbf{Consistent trend on architectural freedom shifting co-active feature pairs toward constructive decoder interactions}. This figure complements~\Cref{fig:topk-architectural-regimes}.}
    \label{fig:architectural-regimes}
\end{figure}

\begin{figure}[H]
    \centering
    \includegraphics[width=\linewidth]{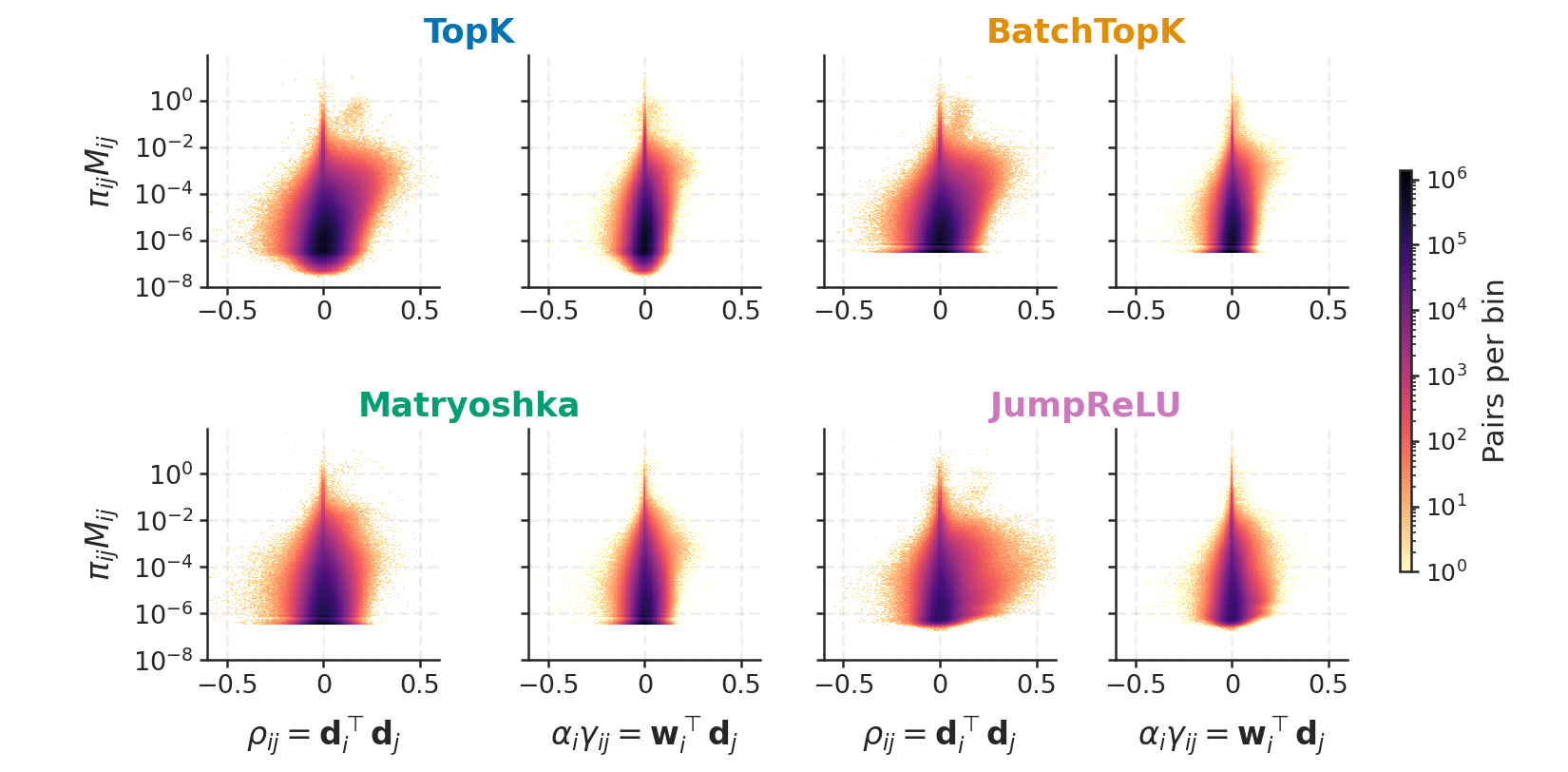}
    \caption{\textbf{The untying mechanism to separate constructive decoder interactions from encoder cross-talk is consistent across SAEs, including TopK, BatchTopK, Matryoshka, and JumpReLU.} This figure complements~\Cref{fig:encdec}.}
    \label{fig:bias-alpha-more}
\end{figure}

\subsection{Interference under varying sparsity and dictionary width}
\label{app:kp-profile}
\vspace{-3mm}
\begin{figure}[H]
    \centering
    \includegraphics[width=\linewidth]{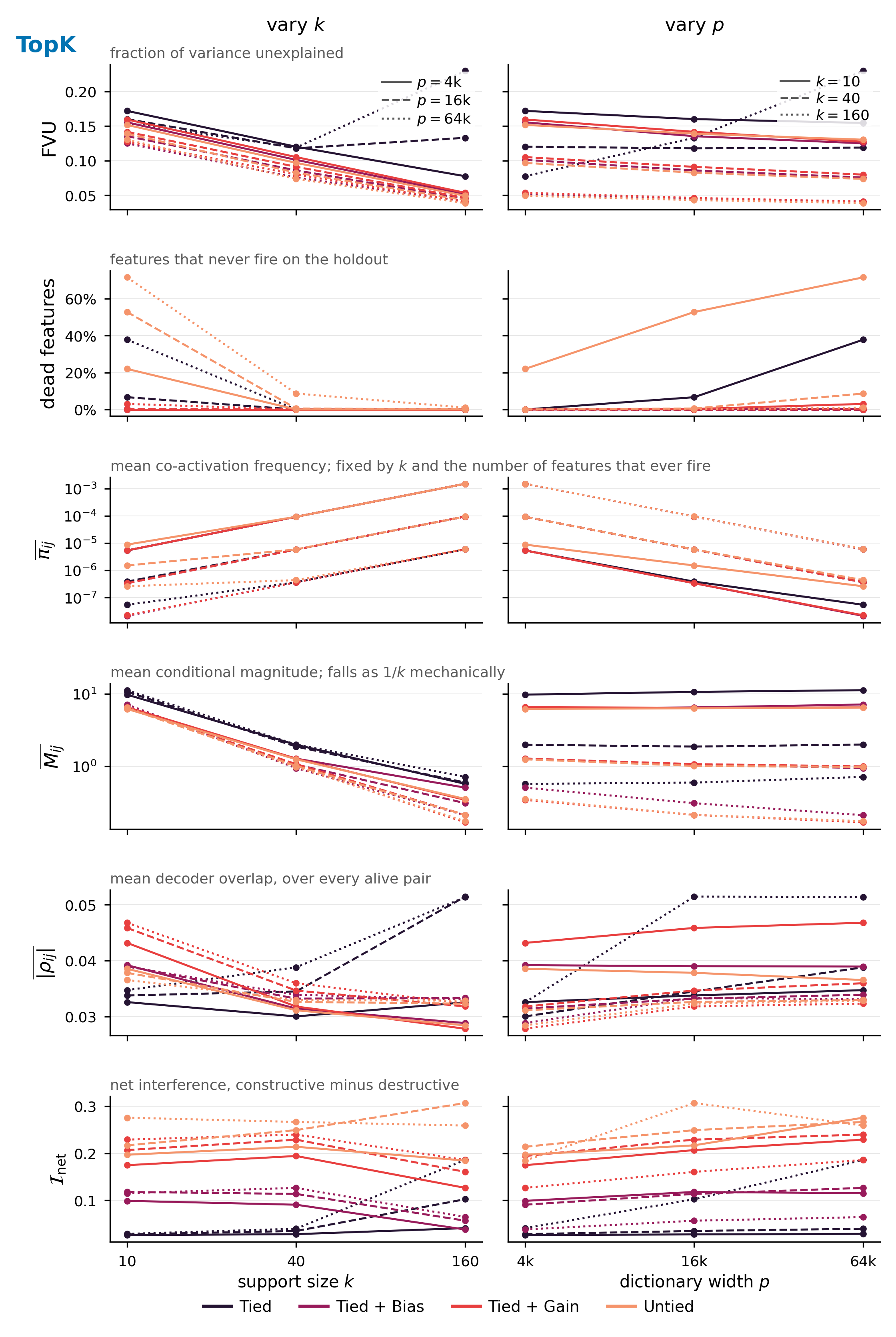}
    \caption{\textbf{Interference across sparsity and dictionary width.} TopK SAEs trained on Pythia-160M, varying constraints, support size $k$, and width $p$. \textbf{Left:} varying $k$ \textbf{Right:} varying $p$.}
    \label{fig:topk-kp-all}
\end{figure}

\newpage
\subsection{Generality of the interference signatures}
\label{app:generality}

We test whether the interference signatures observed in the main experiments generalize beyond our training setup. We consider two settings: a different modality (MNIST; \Cref{app:mnist}) and independently trained language-model SAEs from SAEBench (\Cref{app:saebench}).

\subsubsection{A different modality: MNIST}
\label{app:mnist}

We train SAEs on MNIST with pixels scaled to $[0,1]$, using 50{,}000 training and 10{,}000 validation examples. Models have $p=1{,}000$ features and $k=10$, and are trained for 1{,}000 epochs with Adam at learning rate $10^{-4}$. Decoder columns are kept unit norm, and inputs are centered using the geometric median of the training set \citep{bricken2023monosemanticity,gao2025scaling,bussmann2024batchtopk}. JumpReLU uses an $\ell_0$-target penalty with coefficient $10$ \citep{rajamanoharan2024jumping}; Matryoshka uses nested groups of sizes $250$, $500$, $750$, and $1{,}000$ \citep{bussmann2025learning}.

\begin{figure}[H]
    \centering
    \includegraphics[width=\linewidth]{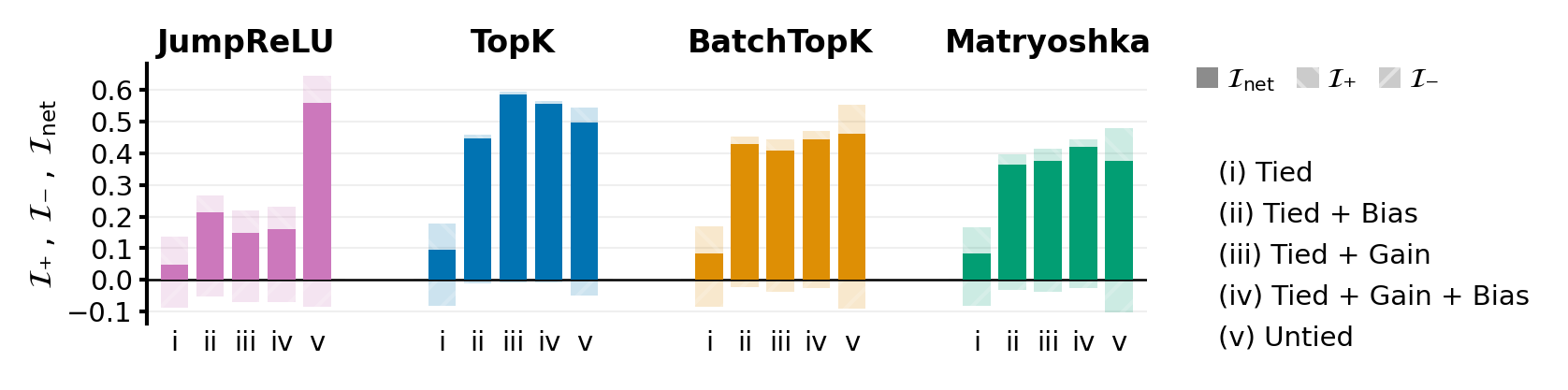}
    \caption{\textbf{Effective interference on MNIST.} The same qualitative pattern as in language models appears: interference is lowest in the most constrained tied setting and increases as architectural degrees of freedom are added. This figure complements~\Cref{fig:aggregate-interference}.}
    \label{fig:aggregate-interference-mnist-signed}
\end{figure}

\begin{figure}
    \centering
    \includegraphics[width=\linewidth]{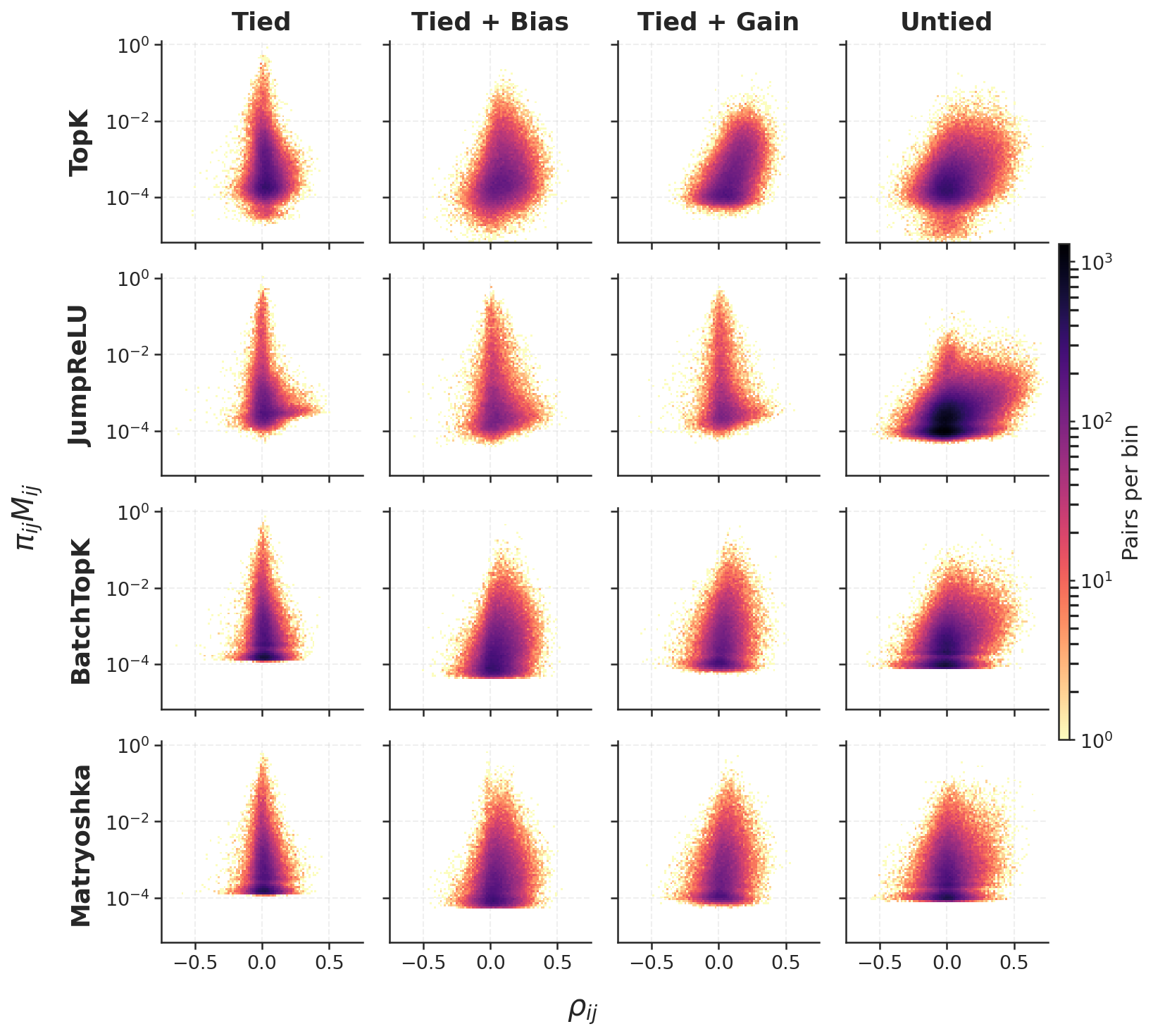}
    \caption{\textbf{Decoder orientation and co-activation on MNIST.} Relaxing the constraints broadens the distribution of decoder correlations and increases co-activation among non-orthogonal feature pairs. This figure complements~\Cref{fig:topk-architectural-regimes}.}
    \label{fig:architectural-regimes-mnist}
\end{figure}

\subsubsection{Independently trained SAEs: SAEBench}
\label{app:saebench}

\begin{wrapfigure}[16]{r}{0.6\textwidth}
    \centering
    \vspace{-7mm}
    \includegraphics[width=0.29\textwidth]{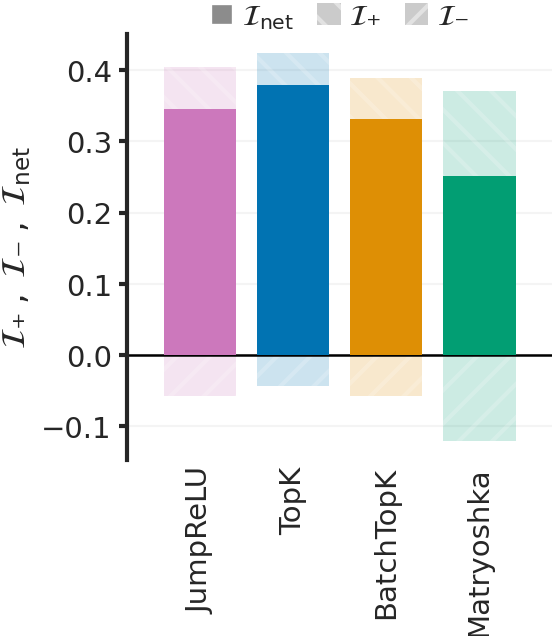}
    \hfill
    \includegraphics[width=0.29\textwidth]{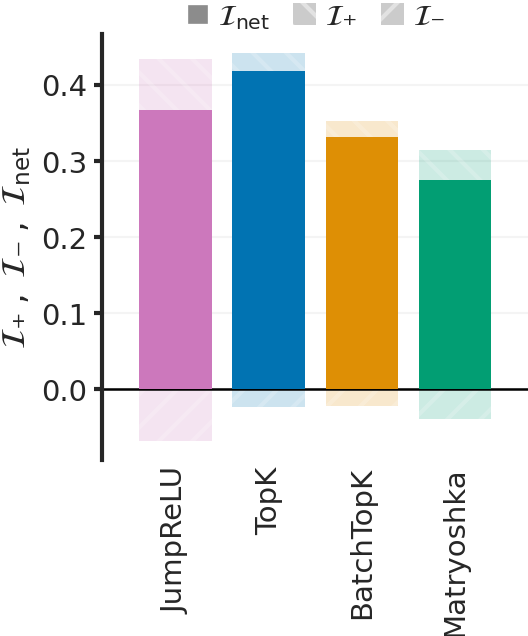}
    \caption{\textbf{Effective interference for independently trained SAEBench SAEs.}
    Gemma-2-2B layer 12 (left) and Pythia-160M layer 8 (right), across the four architectures.}
    \label{fig:saebench-signed}
\end{wrapfigure}

Our main experiments are deliberately controlled: all architectures are trained on the same model, data, width, sparsity, and optimization setup, allowing us to isolate the effect of architectural constraints. To test whether the same interference signatures persist beyond this controlled setting, we also evaluate independently trained SAEs released with SAEBench \citep{karvonen2025saebench}. We use $p=16{,}384$ checkpoints targeting $k=40$ for gemma-2-2b \citep{team2024gemma} layer 12 and Pythia-160M-deduped layer 8, evaluated using SAEBench's own OpenWebText pipeline.

These SAEs differ from ours in several aspects of the training recipe: they are trained for longer, use learning-rate warmup and decay, learn the pre-encoder bias rather than fixing it to the geometric median, and, for the hard-sparsity architectures, include an auxiliary loss designed to revive dead features. They are also available only in the untied setting, so their absolute interference levels are not directly comparable to those in our controlled sweep. Despite these differences, the same qualitative signatures remain visible across both model families (\Cref{fig:saebench-pairs}). In particular, the SAEBench models exhibit substantial effective interference (\Cref{fig:saebench-signed}), generally more than we observe in our own untied runs, suggesting that additional optimization and training choices in these pipelines may themselves contribute to effective interference.

\begin{figure}
    \centering
    \includegraphics[width=\linewidth]{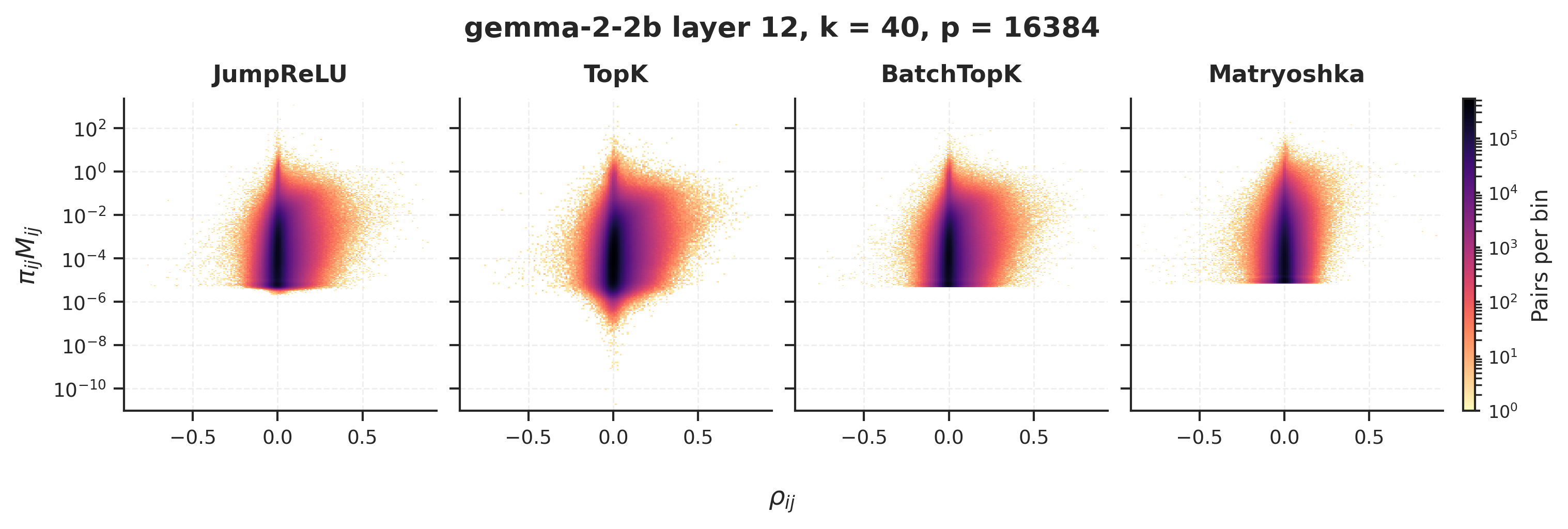}

    \vspace{0.5em}

    \includegraphics[width=\linewidth]{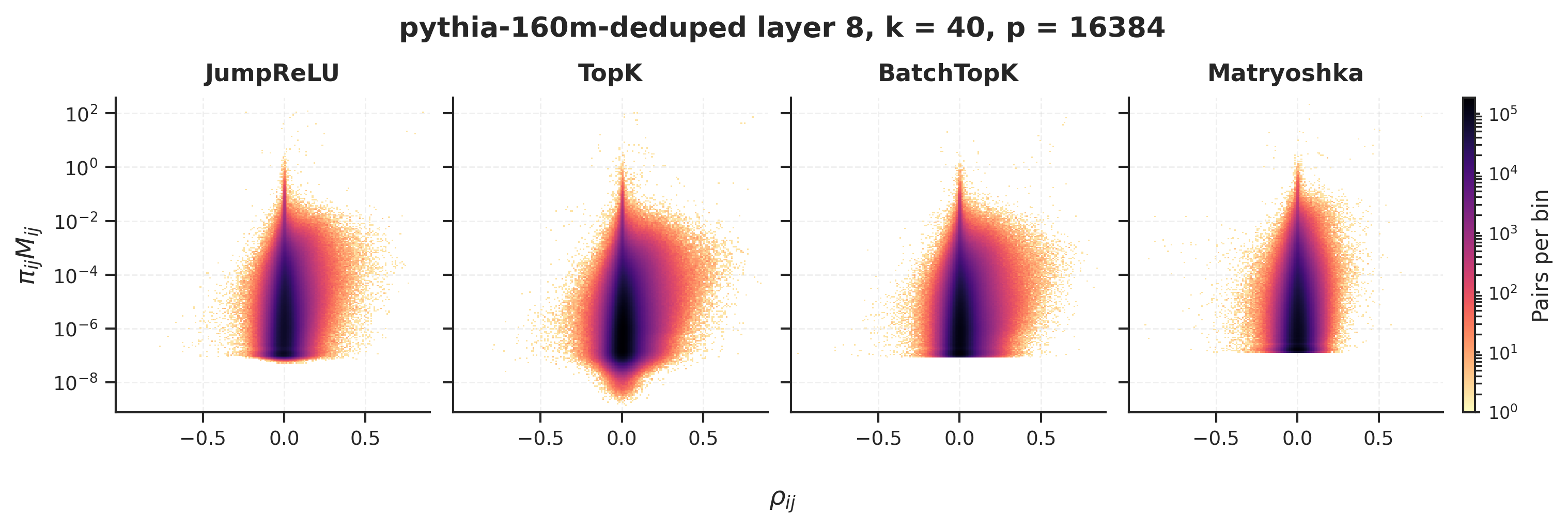}
    \caption{\textbf{Decoder orientation and co-activation for independently trained SAEBench SAEs.}
    Gemma-2-2B layer 12 (top) and Pythia-160M layer 8 (bottom).}
    \label{fig:saebench-pairs}
\end{figure}

\end{document}